\documentclass[twoside]{article}

\usepackage[accepted]{aistats2021}
\usepackage[round]{natbib}
\usepackage[utf8]{inputenc} % allow utf-8 input
\usepackage[T1]{fontenc}    % use 8-bit T1 fonts
\usepackage{hyperref}       % hyperlinks
\usepackage{url}            % simple URL typesetting
\usepackage{booktabs}       % professional-quality tables
\usepackage{amsfonts}       % blackboard math symbols
\usepackage{nicefrac}       % compact symbols for 1/2, etc.
\usepackage{microtype}      % microtypography

\usepackage{graphicx}
\graphicspath{ {figures/} }
\usepackage{bbm}
\usepackage{amsmath}
\usepackage{amsthm}
\usepackage{xcolor}

\usepackage{algorithm}
\usepackage{algorithmic}
\usepackage{float}
\usepackage{color}
\usepackage{amsfonts}
\usepackage{multicol}
\usepackage{wrapfig}
\usepackage{lipsum}
\usepackage{mwe}

\newtheorem{thm}{Theorem}
\newtheorem{lem}{Lemma}

\newcommand{\argmax}{\operatornamewithlimits{argmax}}
\newcommand{\argmin}{\operatornamewithlimits{argmin}}

\newtheorem{assum}{Assumption}

\newtheorem{proposition}{Proposition}%[section]

\newtheorem{remark}{Remark}

\begin{document}

\runningtitle{Online Stochastic Gradient Descent and Thompson Sampling}

\runningauthor{Ding, Hsieh, Sharpnack }

\twocolumn[

\aistatstitle{An Efficient Algorithm For Generalized Linear Bandit: Online Stochastic Gradient Descent and Thompson Sampling}

\aistatsauthor{ Qin Ding \And Cho-Jui Hsieh }

\aistatsaddress{  Department of Statistics \\ University of California, Davis \\ qding@ucdavis.edu 
\And  Department of Computer Science \\ University of California, Los Angeles \\ chohsieh@cs.ucla.edu } 

\aistatsauthor{  James Sharpnack }

\aistatsaddress{  Department of Statistics \\ University of California, Davis \\ jsharpna@ucdavis.edu } ]

% \maketitle

\begin{abstract}
We consider the contextual bandit problem, where a player sequentially makes decisions based on past observations to maximize the cumulative reward. Although many algorithms have been proposed for contextual bandit, most of them rely on finding the maximum likelihood estimator at each iteration, which requires $O(t)$ time at the $t$-th iteration and are memory inefficient. A natural way to resolve this problem is to apply online stochastic gradient descent (SGD) so that the per-step time and memory complexity can be reduced to constant with respect to $t$, but a contextual bandit policy based on online SGD updates that balances exploration and exploitation has remained elusive. In this work, we show that online SGD can be applied to the generalized linear bandit problem. The proposed SGD-TS algorithm, which uses a single-step SGD update to exploit past information and uses Thompson Sampling for exploration, achieves $\tilde{O}(\sqrt{T})$ regret with the total time complexity that scales linearly in $T$ and $d$, where $T$ is the total number of rounds and $d$ is the number of features. Experimental results show that SGD-TS consistently outperforms existing algorithms on both synthetic and real datasets. 
\end{abstract}

\section{INTRODUCTION}\label{intro}
A contextual bandit is a sequential learning problem, where each round the player has to decide which action to take by pulling an arm from $K$ arms. Before making the decisions at each round, the player is given the information of $K$ arms, represented by $d$-dimensional feature vectors. Only the rewards of pulled arms are revealed to the player and the player may use past observations to estimate the relationship between feature vectors and rewards. However, the reward estimate is biased towards the pulled arms as the player cannot observe the rewards of unselected arms. The goal of the player is to maximize the cumulative reward or minimize cumulative regret across $T$ rounds. Due to this partial feedback setting in bandit problems, the player is facing
a dilemma of whether to exploit by pulling the best arm based on the current estimates, or to explore uncertain arms to improve the reward estimates. This is the so-called exploration-exploitation trade-off. Contextual bandit problem has substantial applications in recommender system \citep{li2010contextual}, clinical trials \citep{woodroofe1979one}, online advertising \citep{schwartz2017customer}, etc. It is also the fundamental problem of reinforcement learning \citep{sutton1998introduction}.

The most classic problem in contextual bandit is the stochastic linear bandit \citep{abbasi2011improved,chu2011contextual}, where the expected rewards follow a linear model of the feature vectors and an unknown model parameter $\theta^* \in \mathbbm{R}^d$. 
Upper Confidence Bound (UCB) \citep{abbasi2011improved,auer2002finite,chu2011contextual} and Thompson Sampling (TS) \citep{thompson1933likelihood,agrawal2012analysis,agrawal2013thompson,chapelle2011empirical} are two most popular algorithms to solve bandit problems. UCB uses the upper confidence bound to estimate the reward optimistically and therefore mixes exploration into exploitation. TS assumes the model parameter follows a prior and uses a random sample from the posterior to estimate the reward model. 
Despite the popularity of stochastic linear bandit, linear model is restrictive in representation power and the assumption of linearity rarely holds in practice. This leads to extensive studies in more complex contextual bandit problems such as generalized linear bandit (GLB) \citep{filippi2010parametric,jun2017scalable,li2017provably}, where the rewards follow a generalized linear model (GLM). 
\cite{li2012unbiased} shows by extensive experiments that GLB achieves lower regret than linear bandit in practice.

For most applications of contextual bandit, efficiency is crucial as the decisions need to be made in real time. 
While GLB can still be solved by UCB or TS, the estimate of upper confidence bound or posterior becomes much more challenging than the linear case. It does not have closed form in general and has to be approximated, which usually requires costly operations in online learning.
As pointed out by \cite{li2017provably}, most GLB algorithms suffer from two expensive operations. The first is that they need to invert a $d\times d$ matrix every round, which is time-consuming when $d$ is large. The second is that they need to find the maximum likelihood estimator (MLE) by solving an optimization problem using all the previous observations at each round. This results in $\Omega(T^2)$ time and $O(T)$ memory for $T$ rounds.

From an optimization perspective, stochastic gradient descent (SGD) \citep{hazan2016introduction} is a popular algorithm for both convex and non-convex problems, even for complex models like neural networks. Online SGD \citep{hazan2016introduction} is an efficient optimization algorithm that incrementally updates the estimator via new observations at each round. 
Although it is natural to apply online SGD to contextual bandit problems so that the time complexity at the $t$-th round can be reduced to constant with respect to $t$, it has not been successfully used due to the following reasons: 1) the hardness of constructing unbiased stochastic gradient with controllable variance due to the partial feedback setting in bandit problems, 2) the difficulty to achieve a balance between sufficient exploration and fast convergence to the optimal decision using solely online SGD, 3) lack of theoretical guarantee. Previous attempts of online SGD in contextual bandit problems are limited to empirical studies. \cite{bietti2018contextual} uses importance weight and doubly-robust techniques to construct unbiased stochastic gradient 
with reduced variance. In \cite{riquelme2018deep}, it is shown that the inherit randomness of SGD does not always offer enough exploration for bandit problems. To the best of our knowledge, there is no existing work that can successfully apply online SGD to update the model parameter of a contextual bandit, while maintaining low theoretical regret.

In this work, we study how online SGD can be appropriately applied to GLB problems. To overcome the dilemma of exploration and exploitation, we propose an algorithm that carefully combines online SGD and TS techniques for GLB. The
exploration factor in TS is re-calibrated to make up for the gap between SGD estimator and MLE. Interestingly, we found that by doing so, we can skip the step of inverting matrices. This leads to $O(Td)$ time complexity of our proposed algorithm when $T$ is much bigger than $d$, which is the most efficient GLB algorithm so far. We provide theoretical guarantee of our algorithm and show that under the ``diversity'' assumption (formally defined in Assumption \ref{lambda_f} of Section \ref{problem}), it can obtain $\tilde O(\sqrt{T})$\footnote{$\tilde O$ ignores poly-logarithmic factors.} regret upper bound for finite-arm GLB problems.
Recently, similar ``diversity'' assumptions have been made to analyze the regret bounds of linear UCB (LinUCB) \citep{wu2020stochastic}, greedy algorithms \citep{bastani2020mostly,kannan2018smoothed} or perturbed adversarial bandit setting \citep{kannan2018smoothed}, though none of them improve the efficiency of contextual bandit algorithms, which is one of the most important contributions of our work. We will discuss in Remark \ref{diverse_ass} the comparisons of previous ``diversity'' assumptions and ours.
% This regret upper bound is optimal for finite-arm contextual bandit problems up to logarithmic factors. Moreover, it improves some existing results in \cite{filippi2010parametric,jun2017scalable,li2017provably} by a factor of $\sqrt{d}$ when the number of arms is finite. 

\textbf{Notations:} We use $\theta^*$ to denote the true model parameter. For a vector $x \in \mathbbm{R}^d$, we use $\|x\|$ to denote its $l_2$ norm and $\|x\|_A = \sqrt{x^T A x}$ to denote its weighted $l_2$ norm associate with a positive-definite matrix $A \in \mathbbm{R}^{d\times d}$. We use $\lambda_{\min}(A)$ to denote the minimum eigenvalue of a matrix $A$. % For two matrices $A$ and $B$ 
Denote $[n] := \{1,2,\dots, n\}$ and $f^\prime$ as the first derivative of a function $f$. Finally, we use $\lfloor b \rfloor$ to denote the maximum integer such that $\lfloor b\rfloor \leq b$ and use $\lceil b \rceil$ to denote the minimum integer such that $\lceil b\rceil \geq b$.

\section{RELATED WORK}\label{relatedworks}
In this section, we briefly discuss some previous algorithms in GLB. \cite{filippi2010parametric} first proposes a UCB type algorithm, called GLM-UCB. It achieves $\tilde O(\sqrt{T})$ regret upper bound. According to \cite{dani2008stochastic}, this regret bound is optimal up to logarithmic factors for contextual bandit problems. \cite{li2017provably} proposes a similar algorithm called UCB-GLM. It improves the regret bound of GLM-UCB by a $\sqrt{\log T}$ factor.
The main idea is to calculate the MLE of $\theta^*$ at each round, and then find the upper confidence bound of reward estimates.
The time complexity of these two algorithms depends quadratically on both $d$ and $T$ as they need to calculate the MLE and matrix inverse every round. SupCB-GLM \citep{li2017provably} has similar regret bounds for finite-arm GLB problem. Its theoretical time complexity is similar to UCB-GLM, although it is impractical generally.

Another rich line of algorithms for GLB follows TS scheme, where the key is to estimate the posterior of $\theta^*$ after observing extra data at each round. 
Laplace-TS \citep{chapelle2011empirical} estimates the posterior of regularized logistic regression by Laplace approximations, whose per-round time complexity is $O(d)$. However, Laplace-TS works only for logistic bandit and does not apply to general GLB problems. Moreover, it performs poorly when the feature vectors are non-Gaussian and when $d>K$.
\cite{dumitrascu2018pg} proposes P\'{o}lya-Gamma augmented Thompson Sampling (PG-TS) with a Gibbs sampler to estimate the posterior for logistic bandit. 
However, Gibbs sampler inference is very expensive in online algorithms. The time complexity of PG-TS is $O(M(d^2T^2 + d^3T))$, where $M$ is the burn-in step. In general, previous TS based algorithms for logistic bandit have regret bound $\tilde O(\sqrt{T})$ \citep{dong2019on,abeille2017linear,russo2014learning}.

More recently, \cite{kveton2020randomized} proposed two algorithms for GLB, both enjoy $\tilde O(\sqrt{T})$ total regret. GLM-TSL \citep{kveton2020randomized} follows the TS technique. It draws a sample from the approximated posterior distribution and pulls the arm with the best estimates of this posterior. As it needs to calculate the MLE and the covariance matrix of the posterior needs to be reweighted using previous pulls every round, its time complexity depends quadratically on both $d$ and $T$. 
GLM-FPL \citep{kveton2020randomized} fits a generalized linear model to the past rewards randomly perturbed by the Gaussian noises and pulls the arm that has the best reward based on this model. Its time complexity is also quadratic on $T$. % as it needs to solve the MLE each round. 

In addition to UCB and TS algorithm, $\epsilon$-greedy algorithm \citep{auer2002finite,sutton1998introduction} is also very popular in practice due to its simplicity, although it does not have theoretical guarantee in general bandit framework. At each round, $\epsilon$-greedy has probability $\epsilon$ to randomly pull an arm, and has probability $1-\epsilon$ to pull the best arm from the current estimates. The time complexity of $\epsilon$-greedy algorithm depends quadratically on $T$ as it need to calculate the MLE every round to find the current best estimates.

To make GLB algorithms scalable, \cite{jun2017scalable} proposes Generalized
Linear Online-to-confidence-set Conversion (GLOC) algorithm. GLOC utilizes the exp-concavity of the loss function of GLM and applies online Newton steps to construct a confidence set for $\theta^*$. GLOC and its TS version, GLOC-TS both achieve $\tilde O(\sqrt{T})$ regret upper bound. The total time complexity of GLOC is $O(Td^2)$ due to the successful use of an online second order update. However, GLOC remains expensive when $d$ is large. We show a detailed analysis of time complexity of GLB algorithms in Table \ref{time_compare} of Section \ref{exps}.

\section{PROBLEM SETTING}\label{problem}
% \subsection{Generalized linear bandit (GLB)}
We consider the $K$-armed stochastic generalized linear bandit (GLB) setting. Denote $T$ as the total number of rounds. At each round $t \in [T]$, the player observes a set of contexts including $K$ feature vectors $\mathcal{A}_t := \{x_{t,a} | a\in [K]\} \subset \mathbbm{R}^d$. $\mathcal{A}_t$ is drawn IID from an unknown distribution with $\|x_{t,a}\| \leq 1$ for all $t \in [T]$ and $a\in [K]$, where $x_{t,a}$ represents the information of arm $a$ at round $t$. We make the same regularity assumption as in \cite{li2017provably}, i.e., there exists a constant $\sigma_0 > 0$ such that $\lambda_{\min} \left( \mathbbm{E}\left[ \frac{1}{K} \sum_{a=1}^K x_{t,a} x_{t,a}^T \right] \right) \geq \sigma_0^2$.
% Each feature vector has an associated stochastic reward $y_{t,a}$. 
Denote $y_{t,a}$ as the associated random reward of arm $a$ at round $t$.
% At round $t$, a
After $\mathcal{A}_t$ is revealed to the player, the player pulls an arm $a_t \in [K]$ and only observes the reward associated with the pulled arm, $y_{t,a_t}$. In the following, we denote $Y_t = y_{t,a_t}$ and $X_t = x_{t,a_t}$.

In GLB, the expected rewards follow a generalized linear model (GLM) of the feature vectors and an unknown vector $\theta^* \in \mathbbm{R}^d$, i.e., there is a fixed, strictly increasing link function $\mu : \mathbbm{R}\to \mathbbm{R}$ such that $\mathbbm{E}[y_{t,a}|x_{t,a}] = \mu(x_{t,a}^T \theta^*)$ for all $t$ and $a$. For example, linear bandit and logistic bandit are special cases of GLB with $\mu(x) = x$ and $\mu(x) = 1/(1+e^{-x})$ respectively. Without loss of generality, we assume $\mu(x) \in [0,1]$ and $y_{t,a} \in [0,1]$.\footnote{Rewards in $[0,1]$ is a non-critical assumption, which can be easily removed. In fact, we only need the rewards to have bounded variance for all the analysis to work.} 
We also assume that $Y_t$ follows a sub-Gaussian distribution with parameter $R>0$. Formally, the GLM can be written as 
$Y_t = \mu(X_t^T \theta^*) + \epsilon_{t},$
where $\epsilon_t$ are independent zero-mean sub-Gaussian noises with parameter $R$. We use $\mathcal{F}_t = \sigma (a_1,\dots,a_t, \mathcal{A}_1, \dots, \mathcal{A}_t, Y_1,\dots, Y_t)$ to denote the $\sigma$-algebra generated by all the information up to round $t$. Then we have $\mathbbm E\left[ e^{\lambda\epsilon_t} | \mathcal{F}_{t-1} \right] \leq e^{\frac{\lambda^2 R^2}{2}}$ for all $t$ and $\lambda \in \mathbbm R$.
Denote $a_t^* = \argmax_{a\in [K]} \mu(x_{t,a}^T \theta^*)$ and $x_{t,*} = x_{t,a_t^*}$, the cumulative regret of $T$ rounds is defined as
\begin{equation}
    R(T) = \sum_{t=1}^T \left[\mu(x_{t,*}^T \theta^*) - \mu(X_t^T \theta^*) \right].
\end{equation}
The player's goal is to find an optimal policy $\pi$, such that if the player follows policy $\pi$ to pull arm $a_t$ at round $t$, the total regret $R(T)$ or the expected regret $\mathbbm{E}[R(T)]$ is minimized. Note that $R(T)$ is random due to the randomness in $a_t$.
We make the following mild assumptions similar to \cite{li2017provably}.
\begin{assum}\label{liptz}
$\mu$ is differentiable and there exists a constant $L_\mu > 0$ such that $|\mu^{\prime}| \leq L_{\mu}$. %, $|\mu^{\prime\prime}| \leq M_{\mu}$.
\end{assum}
For logistic link function, Assumption \ref{liptz} holds when $L_{\mu} = \frac{1}{4}$. For linear function, we have $L_{\mu} = 1$.
\begin{assum}\label{mu_prime}
We assume $c_3 > 0$, where $c_{\eta}:= \inf_{\{\|x\|\leq 1, \|\theta - \theta^* \| \leq \eta \}} \mu^{\prime} (x^T \theta)$.
\end{assum}
This assumption is not stronger than the assumption made in \cite{li2017provably} for linear bandit and logistic bandit, as \cite{li2017provably} assumes $c_1 > 0$ and $\frac{c_3}{c_1} \sim O(1)$ in both cases. 

To make sure we can successfully apply online SGD update in bandit problems, we also need the following regularity assumption, which assumes that the optimal arm based on any model parameter $\theta$ has non-singular second moment matrix. This assumption is similar to the regularity assumption made in \cite{li2017provably}, which assumes that the averaged second moment matrices of feature vectors, i.e., $\mathbbm{E}[\frac{1}{K}\sum_{a=1}^K x_{t,a} x_{t,a}^T]$ is non-singular. Assumption \ref{lambda_f} below merely says that the same holds for the optimal arm based on any $\theta$.
\begin{assum}\label{lambda_f}
For a fixed $\theta \in \mathbbm{R}^d$,
let $\tilde X_{\theta,t} = \argmax_{a \in [K]} \theta^T x_{t,a}$
Denote $\Sigma_{\theta} = \mathbbm{E}[\tilde X_{\theta,t} \tilde X_{\theta,t}^T]$ and $\lambda_f = \displaystyle{\inf_{\theta}} \lambda_{\min} (\Sigma_{\theta})$. We assume $\lambda_f$ is a positive constant.
\end{assum}
Intuitively, Assumption \ref{lambda_f} means that based on any model parameter $\theta$, the projection of the optimal arm's feature vector onto any direction has positive probability to be non-zero. 
% i.e., the optimal arm's feature vectors span the whole space $\mathbbm{R}^d$. since that the feature vectors are drawn from a certain distribution and a small perturbation of a vector can form a full-rank matrix. 
In practice, the optimal arms at different rounds are diverse, so it is reasonable to assume that the projections of these random vectors onto any direction are not always a constant zero.
\begin{remark}\label{diverse_ass}
\cite{wu2020stochastic} makes another version of diversity assumption and proposes the LinUCB-d algorithm to utilize the diversity property of contexts. It requires that all arms could be optimal under certain contexts and that the corresponding feature vectors span $\mathbbm R^d$. Our Assumption \ref{lambda_f} is different from the one in \cite{wu2020stochastic} since we do not require that all the arms could be optimal. Moreover, LinUCB-d only works for linear bandit and cannot be generalized to GLB problems easily. Even in the linear case, the time complexity of LinUCB-d depends quadratically on $d$.
\cite{bastani2020mostly} analyzes the greedy algorithm under a diversity assumption, which assumes the covariance matrix of all the feature vectors lying in any half space is positive definite. Our assumption is different from this since we only make the diversity assumption on the optimal arm under different $\theta$, instead of all the feature vectors. We will include the experimental comparisons with $\epsilon$-greedy algorithms for GLB problems in Section \ref{exps} and show that our algorithm significantly outperforms it.
\end{remark}

\section{PROPOSED ALGORITHM}\label{proposed}
In this section, we formally describe our proposed algorithm. The main idea is to use online stochastic gradient descent (SGD) procedure to estimate the MLE and use Thompson Sampling (TS) to explore. 

For GLM, the MLE from $n$ data points $\{X_i, Y_i\}_{i=1}^n$ is %defined by 
%\begin{equation*}
    $\hat\theta_n = \argmax_{\theta} \sum_{i=1}^n \left[ Y_i X_i^T \theta - m (X_i^T \theta) \right],$
%\end{equation*}
where $m^\prime (x) = \mu(x)$. Therefore, it is natural to define the loss function at round $t$ to be
$l_t (\theta) = -Y_t X_t^T \theta + m (X_t^T \theta)$.
Effective algorithms in GLB \citep{abeille2017linear,filippi2010parametric,li2017provably,russo2014learning} have been shown to converge to the optimal action at a rate of $\tilde O(\frac{1}{\sqrt{T}})$. Similarly, we need to ensure that online SGD steps will achieve the same fast convergence rate. This rate is only attainable when the loss function is strongly convex. However, the loss function at a single round is convex but not necessarily strongly convex. To tackle this problem, we aggregate the loss function every $\tau$ steps,
where $\tau$ is a parameter to be specified. 
We define the $j$-th aggregated loss function as
\begin{equation}
    l_{j,\tau} (\theta) = \sum_{s=(j-1)\tau + 1}^{j\tau} -Y_s X_s^T \theta + m (X_s^T \theta).
\end{equation}
Let $\alpha$ be a positive constant, we will show in Section \ref{analysis} that when $\tau$ is appropriately chosen based on $\alpha$, the aggregated loss function of $\tau$ rounds is $\alpha$-strongly convex and therefore fast convergence can be obtained.
The gradient and Hessian of $l_{j,\tau}$ are derived as 
\begin{align}\label{gradient}
    \nabla l_{j,\tau} (\theta) &= \sum_{s=(j-1)\tau + 1}^{j\tau} -Y_s X_s  + \mu (X_s^T \theta) X_s,  \\
    \nabla^2 l_{j,\tau} (\theta) &= \sum_{s=(j-1)\tau + 1}^{j\tau} \mu^\prime (X_s^T \theta) X_s X_s^T. 
\end{align}

In the first $\tau$ rounds of the algorithm, we randomly pull arms. Denote $\hat\theta_t$ as the MLE at round $t$ using previous $t$ observations. We calculate the MLE only once at round $\tau$ and get $\hat\theta_{\tau}$. We keep a convex set $\mathcal C = \{\theta: \| \theta-\hat\theta_{\tau} \| \leq 2\}$. 
We will show in Section \ref{analysis} that when $\tau$ is properly chosen, we have $\| \hat\theta_{t} -\theta^*\| \leq 1$ for all $t\geq \tau$. Therefore, for every $t\geq \tau$, we have $\hat\theta_t \in \mathcal{C}$. 
Denote $\tilde \theta_j$ as the $j$-th updated SGD estimator and let $\tilde \theta_0 = \hat\theta_\tau$.
Starting from round $t = \tau+1$, we update $\tilde\theta_j$ every $\tau$ rounds. Since the minimum of the loss function lies in $\mathcal{C}$, we project $\tilde\theta_j$ to the convex set $\mathcal{C}$ (line 9 of Algorithm \ref{sgd_ts}).
Define $\bar\theta_j = \frac{1}{j} \sum_{q=1}^j \tilde\theta_q$, then $\bar\theta_j$ is treated as the posterior mean of $\theta^*$ and we use TS to ensure sufficient exploration. Specifically, we draw $\theta^{\text{TS}}_j$ from a multivariate Gaussian distribution with mean $\bar\theta_j$ and covariance matrix
\begin{equation}\label{cov}
    A_j = \left( \frac{2c_3 g_1(j) ^2}{\alpha j} +  \frac{2g_2(j)^2}{j} \right) I_d,
\end{equation}
where $g_1(j)$ and $g_2(j)$ are defined as

\begin{align}
g_1 (j) &= \frac{R}{c_1} \sqrt{\frac{d}{2} \log(1+\frac{2j\tau}{d}) + 2\log T} \label{g1}\\
g_2 (j) &= \frac{\tau}{\alpha} \sqrt{1+\log j}. \label{g2}
\end{align}

Previous works \citep{filippi2010parametric,li2017provably,jun2017scalable} in GLB use $V_{t+1}^{-1}$ as the covariance matrix, where $V_{t+1}  = \sum_{s=1}^t X_s X_s^T$. In contrast, we use $\frac{2c_3 g_1(j)^2}{\alpha j} I_d$ to approximate $V_{j\tau + 1}^{-1}$. Meanwhile, the covariance matrix in Equation \ref{cov} has an extra second term, which comes from the gap between the averaged SGD estimator $\bar\theta_j$ and the MLE $\hat\theta_{j\tau}$.
Note that similar to the SGD estimator $\tilde\theta_j$, TS estimator $\theta^{\text{TS}}_j$ is updated every $\tau$ rounds. 
At round $t>\tau$, we will pull arm $a_t = \argmax_{a\in [K]} \mu(x_{t,a}^T \theta^{\text{TS}}_j)$, where $j = \lfloor \frac{t-1}{\tau}\rfloor$.
See Figure \ref{notation} for a brief illustration of the notations.
Since our proposed algorithm employs both techniques from online SGD and TS methods, we call our algorithm SGD-TS.
See Algorithm \ref{sgd_ts} for details.
\begin{figure}[ht]
  \begin{center}
    \includegraphics[width=0.46\textwidth]{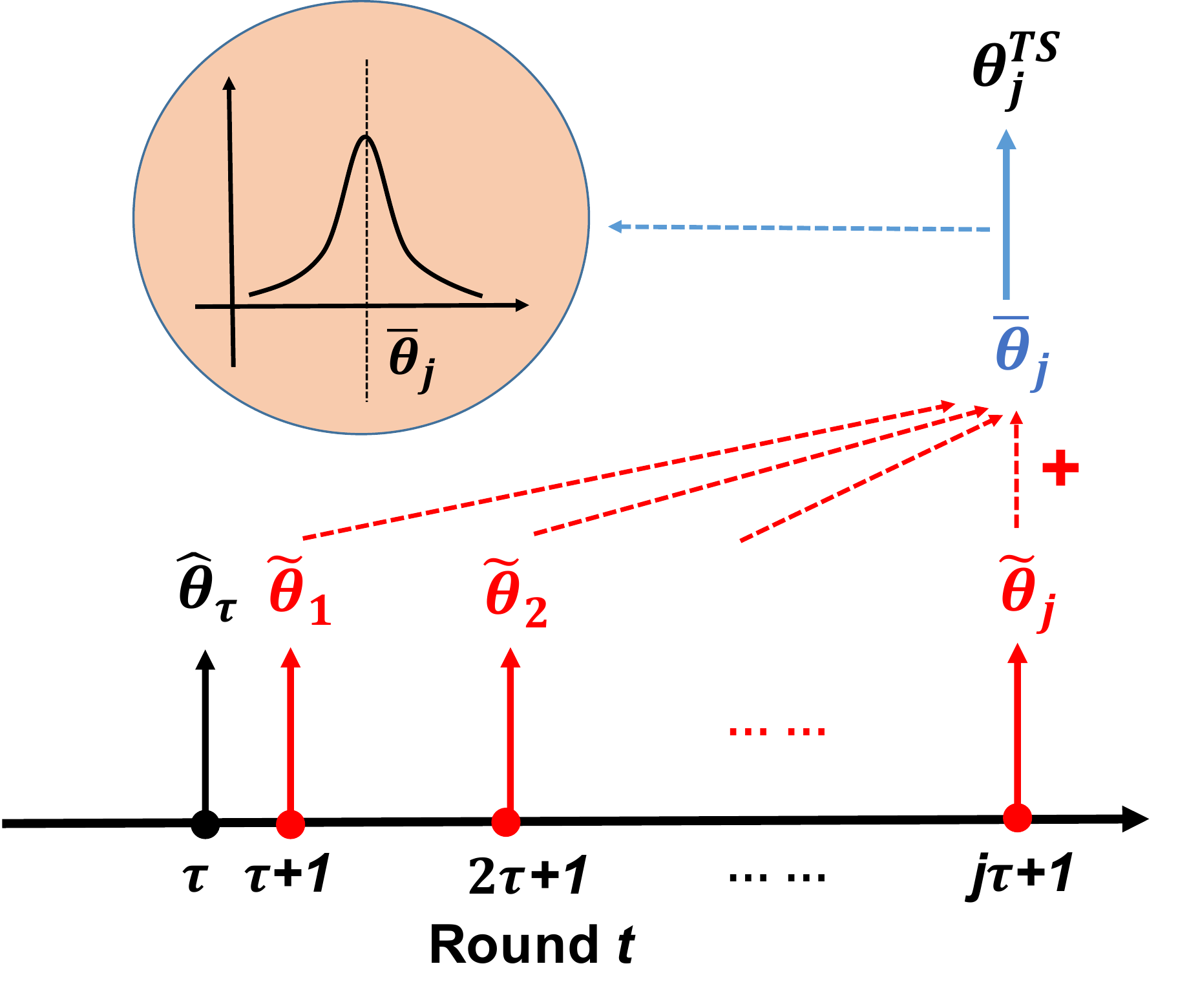}
  \end{center}
  \vspace{-0.18in}
  \caption{Illustration of notations.}
  \label{notation}
\end{figure}

Since some GLB algorithms like UCB-GLM \citep{li2017provably} and GLM-UCB \citep{filippi2010parametric} need to compute MLE every round, to be able to compare the time complexity, we assume the MLE using $t$ datapoints with $d$ features can be solved in $O(td)$ time.
SGD-TS is an extremely efficient algorithm for GLB. 
We only calculate the MLE once at the $\tau$-th round, which costs $O(\tau d)$ time. Then we
update the SGD estimator every $\tau$ rounds and the gradient can be incrementally computed with per-round time $O(d)$. Note that we do not need to calculate matrix inverse every round either since we approximate $V_{t+1}^{-1}$ by a diagonal matrix. In conclusion, the time complexity of SGD-TS in $T$ rounds is $O(Td + d\tau)$, and it will be shown in Section \ref{analysis} that $\tau \sim O(\max\{d,\log T\} / \lambda_f^2)$. In practice, $T$ is usually much greater than $d$, and in such cases, SGD-TS costs $O(Td)$ time.
Our algorithm improves the efficiency significantly if either $d$ or $T$ is large. See Table \ref{time_compare} in Section \ref{exps} for comparisons with other algorithms. 

\begin{algorithm}[ht] 
\caption{Online stochastic gradient descent with Thompson Sampling (SGD-TS)}
\label{sgd_ts}
\textbf{Input}: $T, K, \tau, \alpha$. % , $\mathcal{A} = \{x_{t,a}, t\in[T], a\in [K]\}$.
\begin{algorithmic}[1]
   % \STATE Initialize constant $\tau$, $\alpha$.
   \STATE Randomly choose $a_t \in [K]$ and record $X_t$, $Y_t$ for $t\in [\tau]$.
   \STATE Calculate the maximum-likelihood estimator $\hat\theta_{\tau}$ by solving $\sum_{t=1}^\tau (Y_t - \mu(X_t^T \theta)) X_t = 0$.
   \STATE Maintain convex set $\mathcal C = \{\theta: \| \theta-\hat\theta_{\tau} \| \leq 2\}$.
   \STATE $\tilde\theta_0 \gets \hat\theta_{\tau}$.
   \FOR{$t = \tau+1$ {\bfseries to} $T$}
        \IF{$t \% \tau = 1$}
            \STATE $j\gets \lfloor (t-1) / \tau \rfloor$ and $\eta_j = \frac{1}{\alpha j}$.
            \STATE Calculate $\nabla l_{j,\tau}$ defined in Equation \ref{gradient} %  and 
            \STATE Update
            %\begin{equation*}
            $\tilde\theta_j \gets \prod_{\mathcal{C}} \left( \tilde\theta_{j-1} - \eta_j \nabla l_{j,\tau} (\tilde\theta_{j-1}) \right)
            $. % \end{equation*}
            \STATE Compute $\bar\theta_j = \frac{1}{j} \sum_{q=1}^j \tilde\theta_q$.
            \STATE Compute $A_j$ defined in Equation \ref{cov}.
            \STATE Draw $\theta^{\text{TS}}_{j} \sim \mathcal N \left(\bar\theta_j, A_j \right)$.
        \ENDIF
        \STATE Pull arm $a_t \gets \argmax_{a\in [K]} \mu (x_{t,a}^T \theta^{\text{TS}}_{j} )$ and observe reward $Y_t$.
        % \STATE Observe reward $Y_t$.
        % \STATE $V_{t+1} \gets V_t + X_{t}X_t^T$.
   \ENDFOR
\end{algorithmic}
\end{algorithm}

\section{MATHEMATICAL ANALYSIS}\label{analysis}
In this section, we formally analyze Algorithm \ref{sgd_ts}. Proofs are deferred to supplementary materials.
% In Lemma \ref{close} and Lemma\ref{convergence_sgd}, we analysis the convergence of SGD estimators.
\subsection{Convergence of SGD update}
\begin{lem}\label{close}
Denote $V_{t+1} = \sum_{s=1}^t X_s X_s^T$. If $\lambda_{\min} (V_{t+1}) \geq \frac{16 R^2 [d+\log (\frac{1}{\delta_1})]}{c_1^2},$ where $\delta_1$ is a small probability, then $\|\hat\theta_t - \theta^* \| \leq 1$ holds with probability at least $1-\delta_1$.
\end{lem}

From Lemma \ref{close}, we have $\hat\theta_t \in \mathcal{C}$ with probability at least $1-\delta_1$ 
when $t \geq \tau$ as long as $\tau$ is properly chosen. This is essential because the SGD estimator is projected to $\mathcal{C}$.
In Lemma \ref{convergence_sgd}, we show that when $\tau$ is chosen as Equation \ref{tau}, % the conclusion of Lemma \ref{close} holds for $t\geq \tau$, and 
the averaged SGD estimator $\bar\theta_j$ converges to MLE at a rate of $\tilde O(\frac{1}{\sqrt{j}})$. 

\begin{lem}\label{convergence_sgd}
For a constant $\alpha > 0$, let 
\begin{align}\label{tau}
    \tau_1 &= \left( \frac{C_1 \sqrt{d} + C_2 \sqrt{2\log T}}{\sigma_0^2} \right)^2 + \frac{32 R^2 [d+2\log T]}{c_1^2 \sigma_0^2}, \nonumber\\
    \tau_2 &= \left( \frac{C_1 \sqrt{d} + C_2 \sqrt{3\log T}}{\lambda_f} \right)^2 + \frac{2\alpha}{c_3 \lambda_f}, \nonumber \\
    \tau &= \lceil \max \{ \tau_1, \tau_2\}\rceil,
    % \tau & = \lceil \max \{ ( \frac{C_1 \sqrt{d} + C_2 \sqrt{2\log T}}{\sigma_0^2} )^2 + \frac{32 R^2 [d+2\log T]}{c_1^2 \sigma_0^2},  \nonumber \\
    % & ( \frac{C_1 \sqrt{d} + C_2 \sqrt{3\log T}}{\lambda_f} )^2 + \frac{2\alpha}{c_3 \lambda_f}    \} \rceil,
\end{align}
where $C_1$ and $C_2$ are two universal constants,
then with probability at least $1-\frac{3}{T^2}$, the following holds when $j\geq 1$, % where $\hat\theta_{j\tau}$ is the MLE at round $t=j\tau$.
\begin{equation*}
     \| \bar\theta_j - \hat\theta_{j\tau} \| \leq \frac{\tau}{\alpha} \sqrt{\frac{1+\log j}{j}}.
\end{equation*}
\end{lem}

\subsection{Concentration events}
By the property of MLE and Lemma \ref{convergence_sgd}, we have the concentration property of SGD estimator. % in Lemma \ref{concentrate_e1}.
\begin{lem}\label{concentrate_e1}
Suppose $\tau$ is chosen as in Equation \ref{tau}, and $\alpha \geq c_3$, define $\mathbbm{B}^d_1 = \{x\in \mathbbm{R}^d : \|x\|\leq 1\}$,
we have $E_1 (j)$ holds with probability at least $1-\frac{5}{T^2}$,
where $E_1(j)=\{ x\in \mathbbm{B}^d_1 : |x^T (\bar\theta_j - \theta^*) |    \leq  g_1(j)  \|x\|_{V_{j\tau+1}^{-1}} + g_2 (j) \frac{\|x\|}{\sqrt{j}} \}$ and % is defined in the following,
$g_1(j)$ and $g_2(j)$ are defined in Equation \ref{g1} and Equation \ref{g2}.
% \begin{align*}
%    & E_1 (j) \\
%    & := \{ x\in \mathbbm{B}^d_1 : |x^T (\bar\theta_j - \theta^*) |    \leq  g_1(j)  \|x\|_{V_{j\tau+1}^{-1}} + g_2 (j) \frac{\|x\|}{\sqrt{j}} \}.
% \end{align*}
\end{lem}
The following lemma shows the concentration property of TS estimator.
\begin{lem}\label{ts_concentrate}
Define $u = \sqrt{2 \log (K\tau T^2)}$, we have $\mathbbm{P}(E_2(j) | \mathcal{F}_{j\tau}) \geq 1-\frac{1}{T^2}$, where $E_2(j)$ is defined as the set of all the vectors $x \in \left\{\cup_{t=j\tau+1}^{(j+1)\tau}\mathcal{A}_{t}\right\}$ such that the following inequality holds
\begin{equation*}
     |x^T (\bar\theta_j - \theta^{\text{TS}}_j) | \leq  u \sqrt{  \frac{2c_3 g_1 (j)^2}{\alpha j}\|x\|^2 + 2g_2(j)^2 \frac{\|x\|^2}{j} }.
\end{equation*}
\end{lem}

The above two lemmas show that the TS estimator $\theta^{\text{TS}}_j$ is concentrated around the true model parameter $\theta^*$.
Lemma \ref{anti_concentrate} below offers the anti-concentration property of TS estimator,  which ensures that we have 
enough exploration for the optimal arm. % in bandit problems.
\begin{lem}% [Anti-Concentration]
\label{anti_concentrate}
Denote $j_t = \lfloor \frac{t-1}{\tau}\rfloor$. For any filtration $\mathcal{F}_{t}$ such that $E_1(j_t) \cap \{ \lambda_{\min} (V_{j_t\tau+1}) \geq \frac{\alpha j_t}{c_3} \}$ is true, we have
    $\mathbbm{P}\left( x_{t,*}^T \theta^{\text{TS}}_{j_t}  > x_{t,*}^T \theta^* | \mathcal{F}_{j_t\tau} \right) \geq \frac{1}{4\sqrt{\pi e}}.$
\end{lem}

\subsection{Regret analysis}
Using the concentration and anti-concentration properties of TS estimator in Lemma \ref{concentrate_e1}, \ref{ts_concentrate} and \ref{anti_concentrate},
we are able to bound a single-round regret in Lemma \ref{reg_t}.
Denote $\Delta_i(t) = (x_{t,*} - x_{t,i})^T \theta^*$, $j_t = \lfloor \frac{t-1}{\tau}\rfloor$ and 
\begin{align}\label{ht}
    H_i (t) & =  g_1(j_t)  \|x_{t,i}\|_{V_{j_t\tau+1}^{-1}} + g_2 (j_t) \frac{\|x_{t,i}\|}{\sqrt{j_t}} \nonumber \\
    & +  u \sqrt{ \frac{2c_3 g_1 (j_t)^2}{\alpha j_t} \|x_{t,i}\|^2 + 2g_2(j_t)^2 \frac{\|x_{t,i}\|^2}{j_t} }.
\end{align}
\begin{lem}\label{reg_t}
At round $t \geq \tau$, where $\tau$ is defined in Equation \ref{tau}, denote $E_3(j_t) = \{\lambda_{\min} (V_{j_t\tau+1}) \geq \frac{\alpha j_t}{c_3}\}$, we have % for any possible filtration $\mathcal{F}_{t}$,
\begin{align}
    & \mathbbm{E} [
    \Delta_{a_t} (t) \mathbbm{1}(E_1(j_t) \cap E_2(j_t) \cap E_3(j_t)) ] \nonumber \\
    & \leq \left(1 + \frac{2}{\frac{1}{4\sqrt{\pi e}} - \frac{1}{T^2}} \right) \mathbbm{E} \left[  H_{a_t} (t) \mathbbm{1}(E_3(j_t) )  \right ].
\end{align}
\end{lem}
We are now ready to put together the above information and prove the regret bound of Algorithm \ref{sgd_ts}.
\begin{thm}\label{main_thm}
When Algorithm \ref{sgd_ts} runs with $\alpha = \max\{c_3, d, \log T\} / \lambda_f$, and $\tau$ defined in Equation \ref{tau}, the expected total regret satisfies the following inequality
\begin{align*}
& \mathbbm E [R(T)] \leq  \tau + \frac{7}{T} + L_{\mu} p \sqrt{\tau T} \left[2\sqrt{\frac{c_3}{\alpha}} g_1(J) + 2 g_2(J)\right] \\
& + L_{\mu} p \sqrt{\tau T} u \sqrt{ \frac{2c_3 g_1(J)^2}{\alpha} + 2g_2(J)^2} \sqrt{1+\log \lfloor \frac{T}{\tau} \rfloor} ,
    % \mathbbm E [R(T)] \leq & \tau + L_{\mu} p \sqrt{\tau T} \left[2\sqrt{\frac{c_3}{\alpha}} g_1(J) + 2 g_2(J) + u \sqrt{ \frac{2c_3 g_1(J)^2}{\alpha} + 2g_2(J)^2} \sqrt{1+\log \lfloor \frac{T}{\tau} \rfloor}\right] + \frac{7}{T},
\end{align*}
where $u = \sqrt{2 \log (K\tau T^2)}$, $p= 1 + \frac{2}{\frac{1}{4\sqrt{\pi e}} - \frac{1}{T^2}}$ and $J = \lfloor\frac{T}{\tau} \rfloor$.
\end{thm}

\begin{remark}
Combining the choices of $\tau, \alpha$ and the definition of $g_1(J), g_2(J)$ in Equation \ref{g1}, \ref{g2}, we have $ \mathbbm{E} [R(T)] \sim \tilde O(\sqrt{T})$. To study the dependence of regret bounds on $d$, we use a common condition in the literature (e.g., \cite{li2017provably}) that $\sigma_0^2 \sim O(1)$ and make a similar assumption that $\lambda_f\sim O(1)$. As pointed out by the reader, this is unrealistic and a more proper assumption should be $\sigma_0^2, \lambda_f \sim O(1/d)$. We will discuss more about the dependencies on $d$ in Section \ref{d} in Appendix. In addition to the $\tilde O(\sqrt{T})$ theoretical guarantee of regret upper bound, our algorithm significantly improves efficiency when either $T$ or $d$ is large for GLB. To the best of our knowledge, it is by far the most efficient algorithm for GLB.
See Table \ref{time_compare} in Section \ref{exps} for the comparisons of time complexity with other algorithms.\footnote{Sherman–Morrison formula improves the time complexity of a matrix inverse in UCB-GLM and GLOC to $O(d^2)$.}
Moreover, the memory cost for UCB-GLM, GLM-TSL, SupCB-GLM and $\epsilon$-greedy algorithms is linear in the total time horizon $T$, which could be very large in practice. For our proposed algorithm SGD-TS, the memory cost is a constant with respect to $T$.
\end{remark}

\begin{figure*}[ht]
\includegraphics[width = 0.33\textwidth]{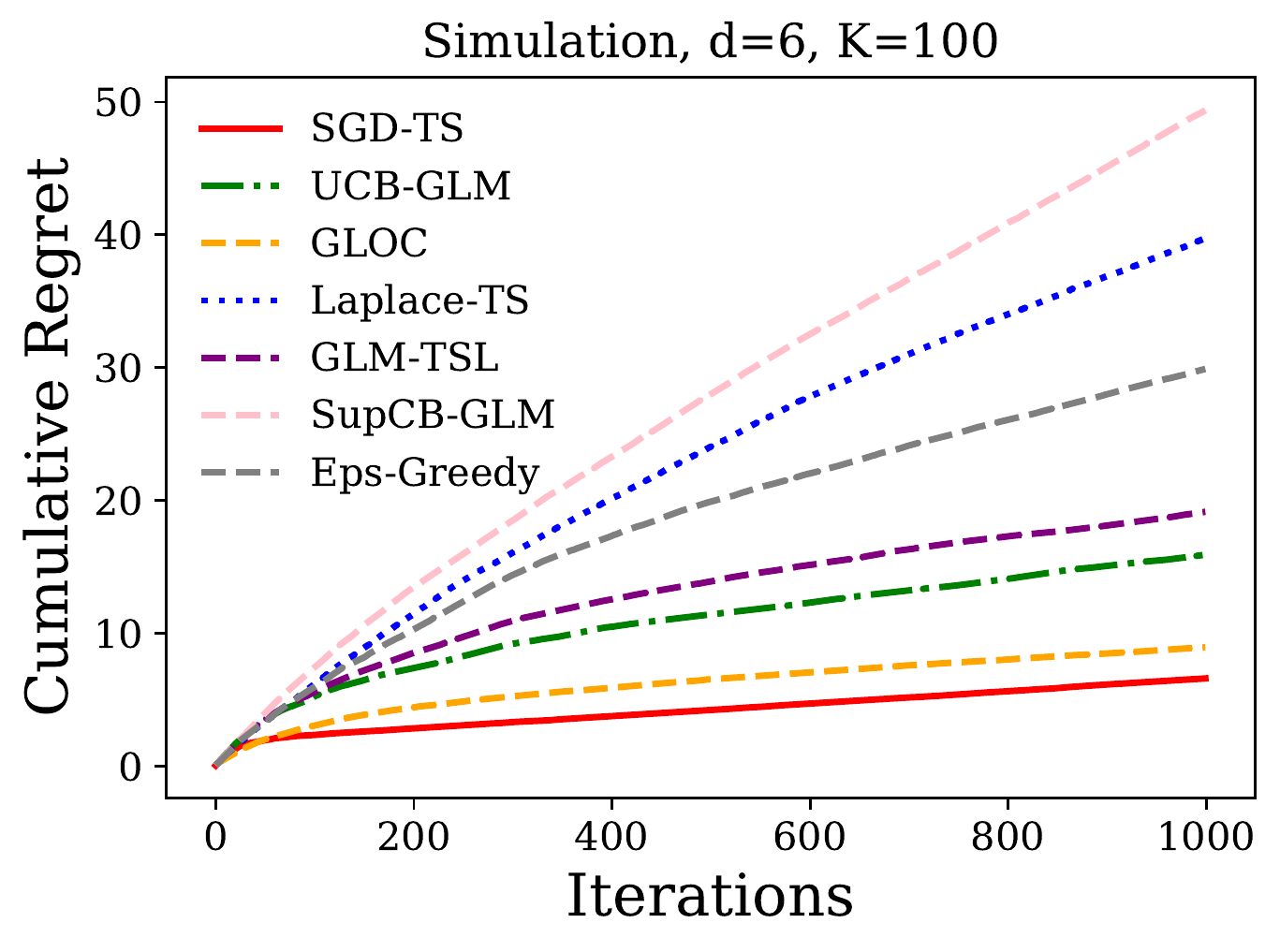}\includegraphics[width = 0.33\textwidth]{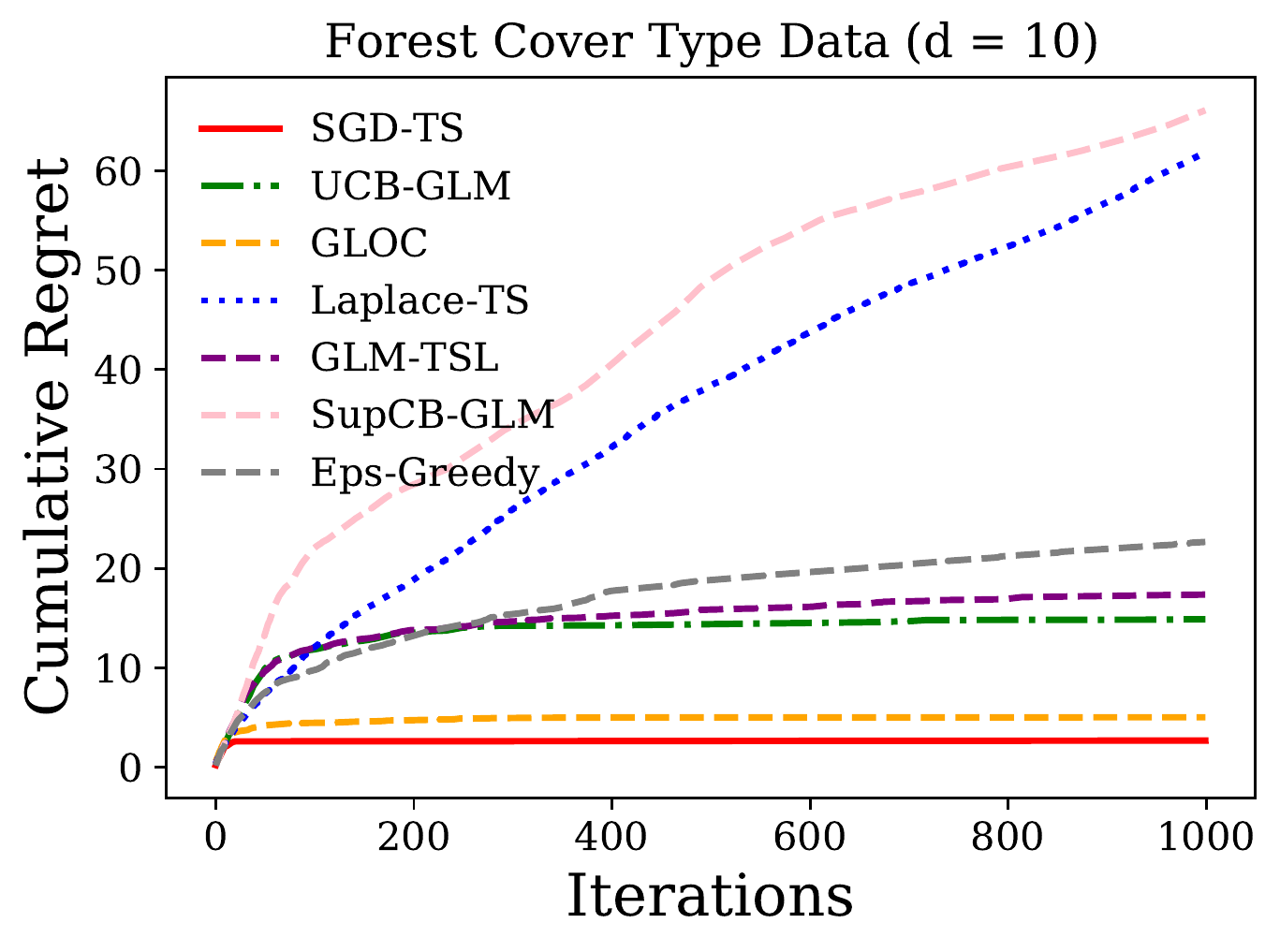}\includegraphics[width = 0.34\textwidth]{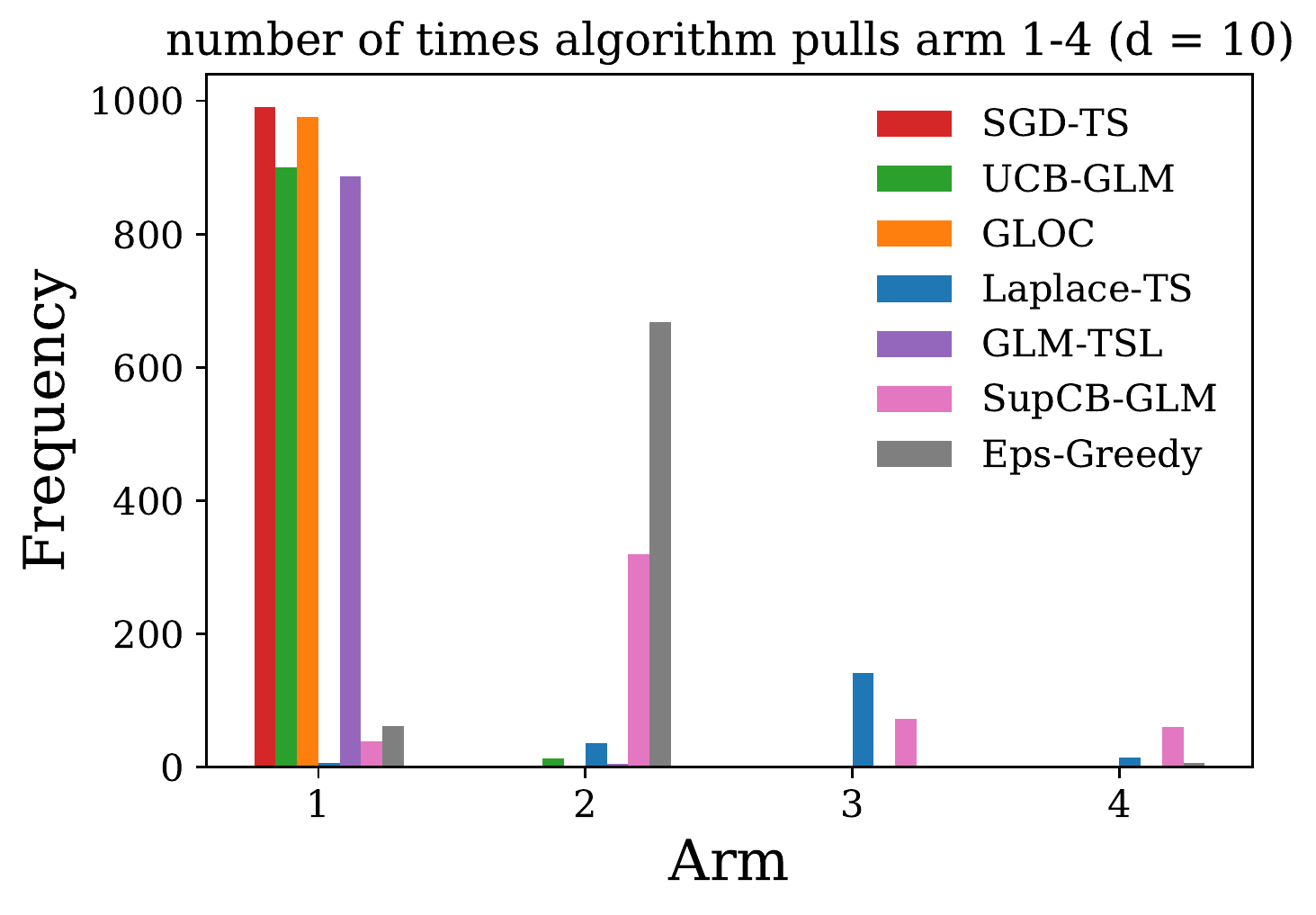}
    
\includegraphics[width = 0.33\textwidth]{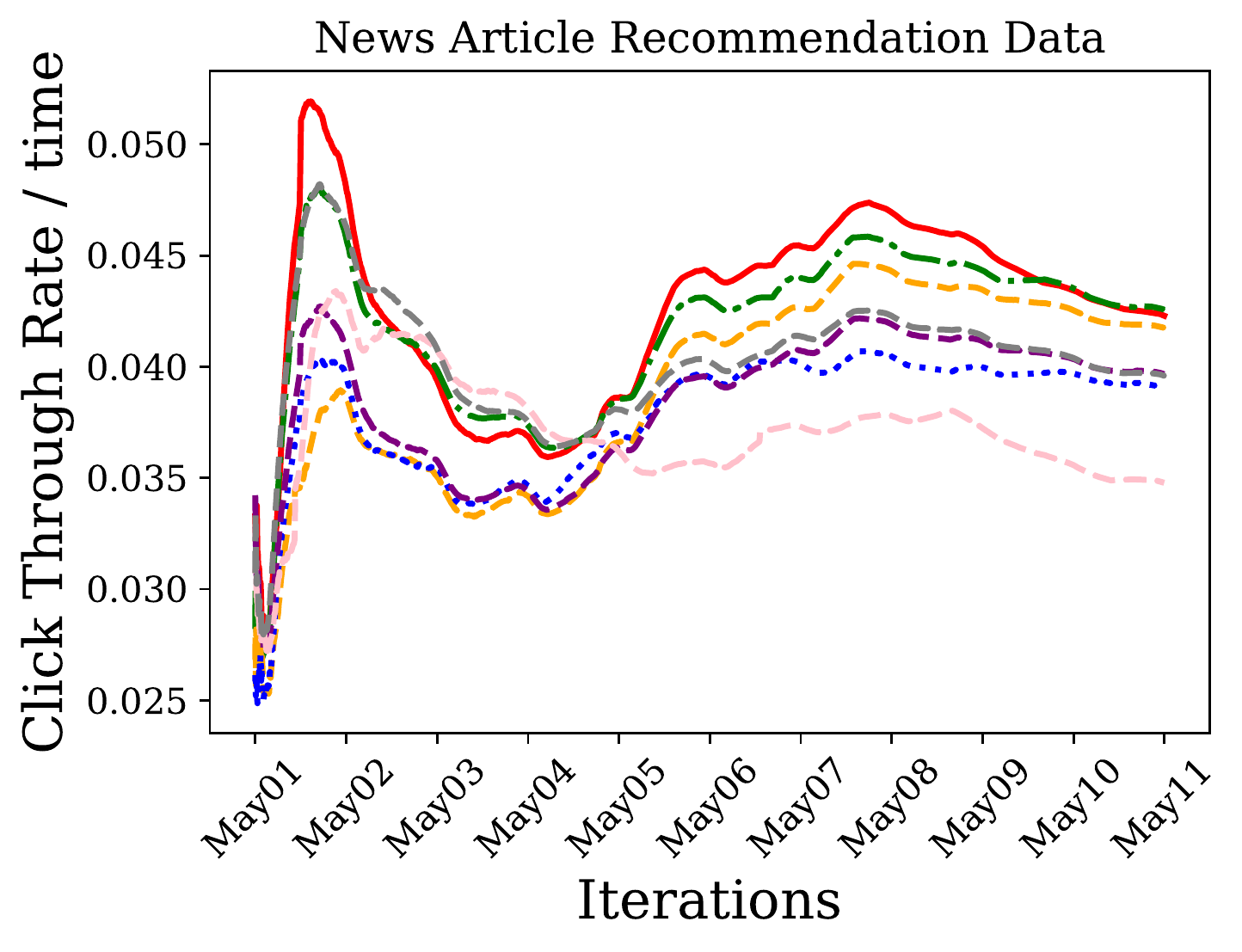}\includegraphics[width = 0.33\textwidth]{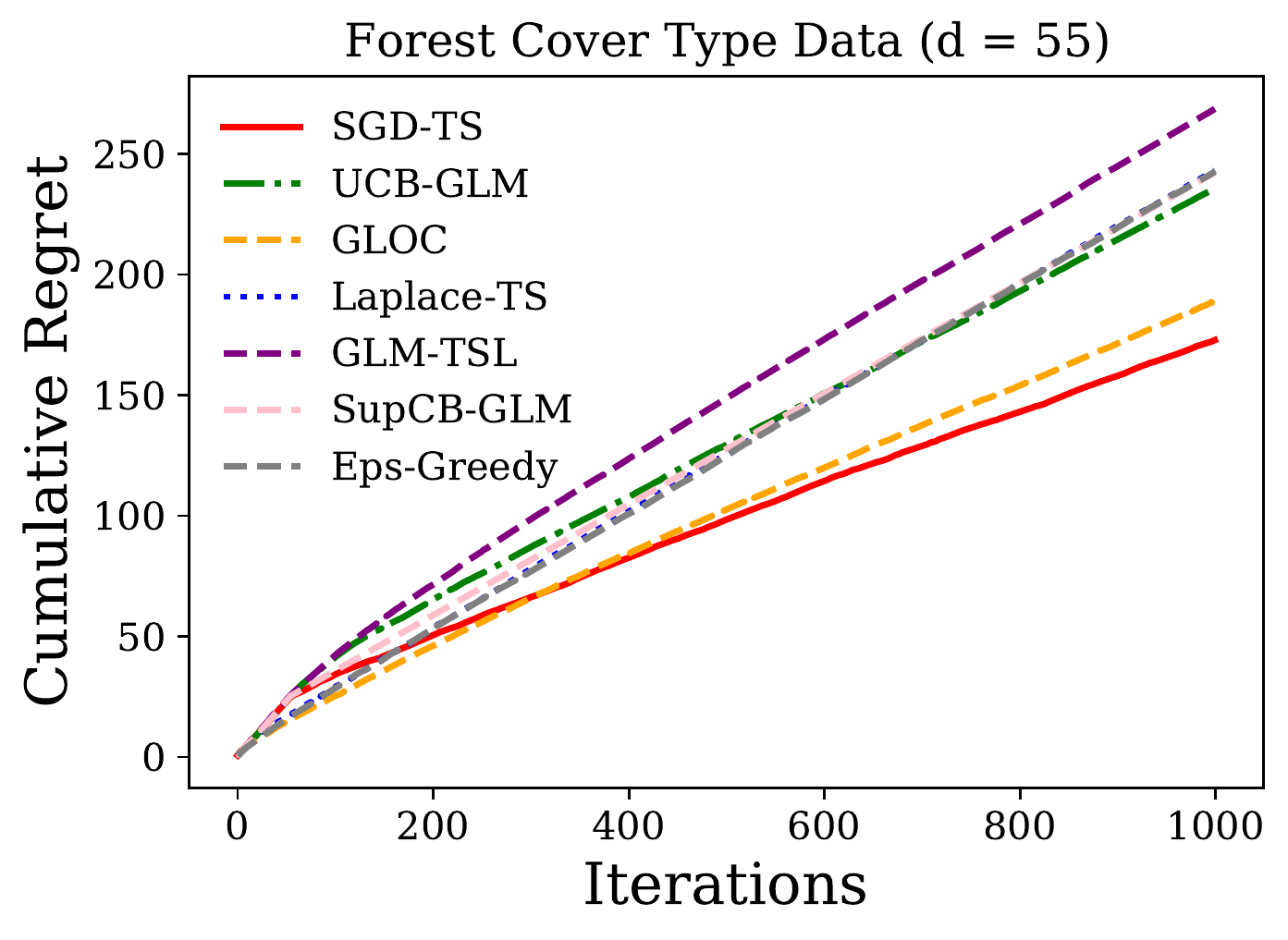}\includegraphics[width = 0.34\textwidth]{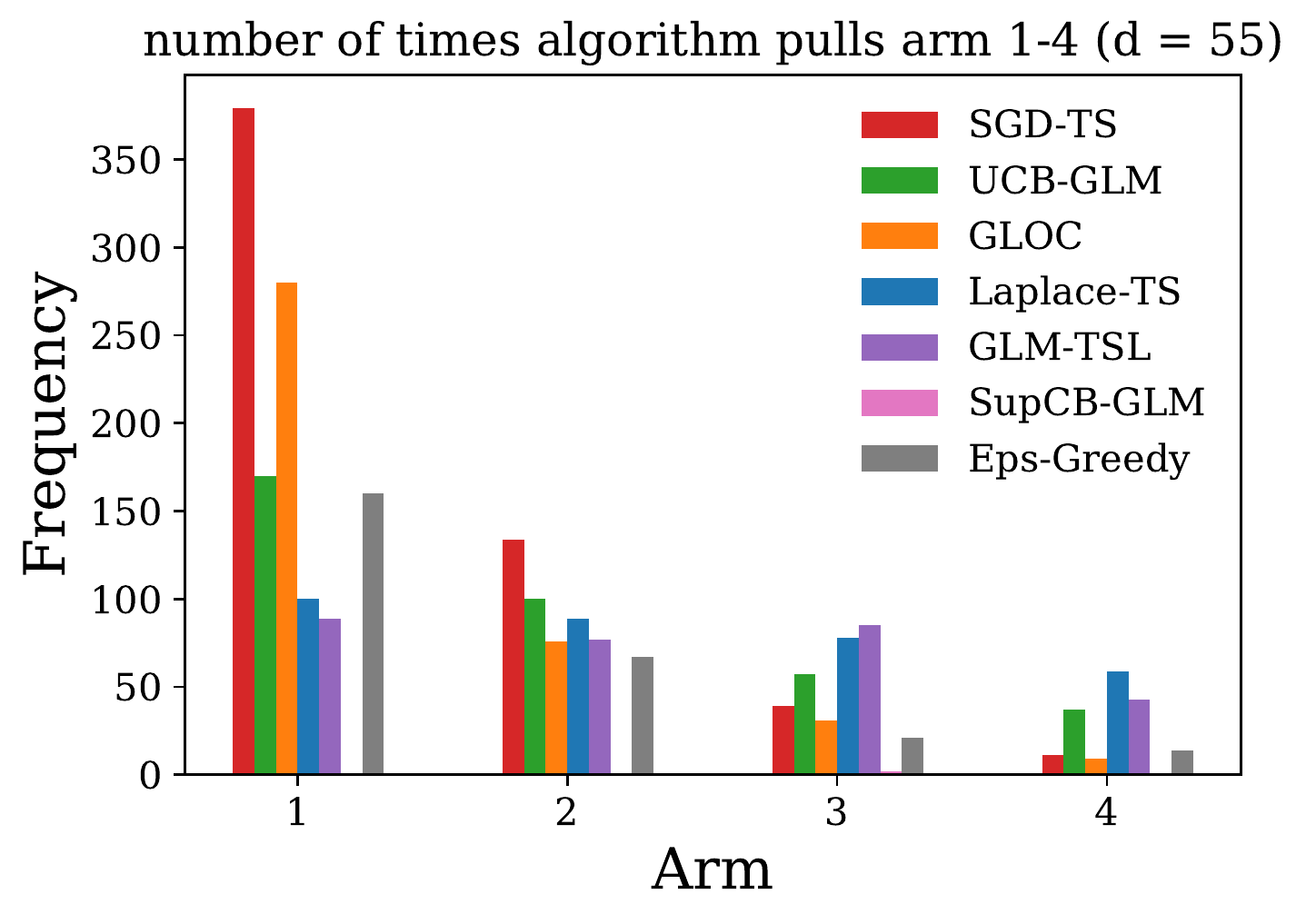}
\vspace{-0.1in}
\caption{For the plots in the first two columns, from left to right, top to bottom, they are for simulation $(d=6, K=100)$, scenario 1 for forest cover type data, news article recommendation data, scenario 2 for forest cover type data respectively. For the plots in the third column, they are the plots of the median frequencies of an algorithm pulls the best 4 arms for scenario 1 and 2 for forest cover type data respectively. (To reduce clutter, the legend in news article recommendation plot is omitted.)}
\label{exp_fig}
\end{figure*}

\section{EXPERIMENTAL RESULTS}\label{exps}
In this section, we show by experiments in both synthetic and real datasets that our proposed SGD-TS algorithm outperforms existing approaches. We compare SGD-TS with UCB-GLM \citep{li2017provably}, 
Laplace-TS \citep{chapelle2011empirical}, GLOC \citep{jun2017scalable}, GLM-TSL \citep{kveton2020randomized}, SupCB-GLM \citep{li2017provably} and $\epsilon$-greedy \citep{auer2002finite,sutton1998introduction}.\footnote{We choose UCB-GLM and GLOC since they have lower theoretical regrets than GLM-UCB and GLOC-TS respectively. We choose GLM-TSL over GLM-FPL since it was shown by \cite{kveton2020randomized} that GLM-TSL enjoys lower regret in practice.}
In order to have a fair comparison, we perform a grid search for the parameters of different algorithms and select the best parameters to report. The covariance matrix in Equation \ref{cov} is set to $A_j = \frac{2a_1^2 + 2a_2^2}{j} I_d$, where $a_1$ and $a_2$ are explorations rates. We do a grid search for exploration rates of SGD-TS, GLOC, GLM-TSL, SupCB-GLM and UCB-GLM in $\{0.01, 0.1, 1, 5, 10\}$. The exploration probability of $\epsilon$-greedy algorithm is set to $\frac{a}{\sqrt{t}}$ at round $t$ and $a$ is also tuned in $\{0.01, 0.1, 1, 5, 10\}$. As suggested by \cite{li2017provably}, $\tau$ should also be treated as a tuning parameter.
For UCB-GLM, GLM-TSL, SupCB-GLM and SGD-TS, we set $\tau = \lfloor C \times \max(\log T, d) \rfloor$ and $C$ is tuned in $\{1,2,\dots,10\}$. The initial step sizes $\eta$ for SGD-TS, GLOC and Laplace-TS are tuned in $\{0.01, 0.05, 0.1, 0.5, 1, 5, 10\}$. In SGD-TS, we set %the step size for SGD-TS at the $j$-th update as
$\eta_j = \frac{\eta}{j}$. The experiments are repeated for 10 times and the averaged results are presented.

\subsection{Simulation}
We simulate a dataset with $T=1000$, $K=100$ and $d=6$. The feature vectors and the true model parameter are drawn IID from uniform distribution in the interval of $[-\frac{1}{\sqrt{d}}, \frac{1}{\sqrt{d}}]$. We build a logistic model on the dataset and draw random rewards $Y_t$ from a Bernoulli distribution with mean $\mu(X_t^T \theta^*)$. As suggested by \cite{dumitrascu2018pg}, Laplace approximation of the global optimum does not always converge in non-asymptotic settings. %It also struggles when feature vectors are non-Gaussian. 
\cite{jun2017scalable} points out that SupCB-GLM is an impractical algorithm.
From Figure \ref{exp_fig}, we can see that our proposed SGD-TS performs the best, while SupCB-GLM and Laplace-TS perform poorly as expected.  

\begin{table*}[ht]
\caption{Comparison of time complexity and averaged runtime with other algorithms. The time complexity listed here assumes that $T$ is much bigger than $d$. GLOC and Laplace-TS (only works for logistic bandit) need to solve an optimization problem on one datapoint every round and we assume this optimization problem can be solved in fixed iterations every round.}
\vspace{0.2in}
\label{time_compare}
\centering
\begin{tabular}{cccc}
\toprule
Algorithms                               & Time Complexity          & Simulation    & Yahoo \\
\midrule
UCB-GLM \citep{li2017provably}            & $O(T^2d)$                & 2.024          & 29.643\\
Laplace-TS (only for logistic bandit) \citep{chapelle2011empirical}  & $O(Td)$  & 0.964          & 27.786\\
GLOC \citep{jun2017scalable}              & $O(Td^2)$                & 0.330          & 0.351\\
GLM-TSL \citep{kveton2020randomized}      & $O(T^2d^2)$          & 8.253          & 81.580\\
SupCB-GLM \citep{li2017provably}          & $O(T^2d)$          & 4.609          & 26.842 \\
$\epsilon$-greedy \citep{auer2002finite,sutton1998introduction}   & $O(T^2d)$      & 2.020          & 35.901\\
\textbf{SGD-TS (This work)}              & $\textbf{O(Td)}$                 & \textbf{0.099} & \textbf{0.127} \\
\bottomrule
\end{tabular}
\end{table*}

\subsection{News article recommendation data}
We compare the algorithms on the benchmark Yahoo! Today Module dataset. This dataset contains $45,811,883$ user visits to the news articles website - Yahoo Today Module from May 1, 2009 to May 10, 2009. For each user's visit, the module will select one article from a changing pool of around $20$ articles to present to the user. The user will decide to click (reward $Y_t=1$) or not to click ($Y_t=0$).
Both the users and the articles are associated with a $6$-dimensional feature vector (including a constant feature), constructed by conjoint analysis with a bilinear model \citep{chu2009case}. We treat the articles as arms and discard the users' features. 
The click through rate (CTR) of each article at every round is calculated using the average of recorded rewards at that round.
We still build logistic bandit on this data.
Each time, when the algorithm pulls an article, the observed reward $Y_t$ is simulated from a Bernoulli distribution with mean equal to its CTR. For better visualization, we plot $\frac{1}{t}\sum_{s=1}^t \mathbbm{E}[Y_s]$ against $t$. Since we want higher CTR, the result will be better if $\frac{1}{t}\sum_{s=1}^t \mathbbm{E}[Y_s]$ is bigger. From the plot in Figure \ref{exp_fig}, we can see that SGD-TS performs better than UCB-GLM during May 1 - May 2 and May 5 - May 9. During other days, UCB-GLM and SGD-TS have similar behaviors. However, other algorithms perform poorly in this real application.

\subsection{Forest cover type data}
We compare the algorithms on the Forest Cover Type data from the UCI repository. The dataset contains $581,021$ datapoints from a forest area. The labels represent the main species of the cover type. For each datapoint, if it belongs to the first class (Spruce/Fir species), we set the reward of this datapoint to $1$, otherwise, we set it as $0$. We extract the features (quantitative features are centralized and standardized) from the dataset and then partition the data into $K=32$ clusters (arms). The reward of each cluster is set to the proportion of datapoints having reward equal to $1$ in that cluster. Since the observed reward is either $0$ or $1$, we build logistic bandits for this dataset. Assume arm 1 has the highest reward and arm 4 has the 4-th highest reward.
We plot the averaged cumulative regret and the median frequencies of an algorithm pulls the best $4$ arms for the following two scenarios in Figure \ref{exp_fig}. 

\textbf{Scenario 1:} Similar to \cite{filippi2010parametric}, we use only the $10$ quantitative features and treat the cluster centroid as the feature vector of the cluster. % So we have a $32$-armed contextual bandit problem with $d = 10$. 
The maximum reward of the $32$ arms is around $0.575$ and the minimum is around $0.005$.  

\textbf{Scenario 2:} To make the classification task more challenging, we utilize both categorical and quantitative features, i.e., $d=55$. Meanwhile, the feature vector of each cluster at each round is a random sample from that cluster. This makes the features more dynamic and the algorithm needs to do more exploration before being able to identify the optimal arm.
The maximum reward is around $0.770$ and the minimum is $0$.  

From the plots, we can see that in both scenarios, our proposed algorithm performs the best and it pulls the best arm most frequently. For scenario 1, GLOC, UCB-GLM and GLM-TSL perform relatively well, while the other algorithms are stuck in sub-optimal arms. This is consistent with the results in \cite{dumitrascu2018pg}.  
For the more difficult scenario 2, SGD-TS is still the best algorithm. GLOC performs relatively well, but it is not able to pull the best arm as frequently as SGD-TS. 
All the other algorithms perform poorly and frequently pull sub-optimal arms. 
% Note that in scenario 2, it costs UCB-GLM and GLOC much more time than scenario 1 to update the decisions, as they need to invert a $d\times d$ matrix every round.

\subsection{Computational cost}
We present the averaged runtime of each algorithm for the simulation and Yahoo news article recommendation in Table \ref{time_compare}. Presented results are the averaged runtime of one repeated experiment for one parameter combination in the grid search set. Note that all algorithms need to solve an optimization problem or invert a matrix each round except our algorithm. For example, UCB-GLM, GLM-TSL, SupCB-GLM and $\epsilon$-greedy need to find MLE every round. Laplace-TS and GLOC need to solve an optimization problem on one data point every round. UCB-GLM, GLM-TSL, SupCB-GLM and GLOC need to calculate matrix inverse every round.
For our proposed SGD-TS, since we only perform a single-step SGD update every round and do not need to calculate matrix inverse, so the real computational cost is the cheapest.

\section{CONCLUSION AND FUTURE WORK}\label{conclusion}
In this paper, we derive and analyze SGD-TS, a novel and efficient algorithm for generalized linear bandit. The time complexity of SGD-TS scales linearly in both total number of rounds and feature dimensions in general. Under the ``diversity'' assumption, we prove a regret upper bound of order $\tilde O(\sqrt{T})$ for SGD-TS algorithm in generalized linear bandit problems. Experimental results of both synthetic and real datasets show that SGD-TS consistently outperforms other state-of-the-art algorithms. To the best of our knowledge, this is the first attempt that successfully applies online stochastic gradient descent steps to contextual bandit problems with theoretical guarantee. Our proposed algorithm is also the most efficient algorithm for generalized linear bandit so far.

\paragraph{Future work} Although generalized linear bandit is successful in many cases, there are many other models that are more powerful in representation for contextual bandit. This motivates a number of works for contextual bandit with complex reward models \citep{chowdhury2017kernelized,riquelme2018deep,zhou2019neural}. 
For most of these works, finding the posterior or upper confidence bound remains an expensive task in online learning. 
While we have seen in this work that online SGD can be successfully applied to GLB under certain assumptions,
it is interesting to investigate whether we could further use online SGD to design efficient and theoretically solid methods for contextual bandit with more complex reward models, like neural networks, etc.

\section*{Acknowledgements}
This research was partially supported by NSF HDR TRIPODS grant CCF-1934568, NSF IIS-1901527 and NSF IIS-2008173. JS is partially supported by NSF DMS 1712996. QD wants to thank Yi-Wei Liu for his useful comments. The authors also want to thank Kwang-Sung Jun for his helpful discussions.

% \clearpage

\bibliographystyle{plainnat}
\bibliography{main}

\onecolumn
\section{SUPPLEMENTARY MATERIAL}
\subsection{Proof of Lemma \ref{close}}
The proof of Lemma \ref{close} is an adaptation from the proof of Theorem 1 in \cite{li2017provably}. 
\begin{proof}
Define $G(\theta) := \sum_{s=1}^t (\mu(X_s^T\theta) - \mu(X_s^T\theta^*) ) X_s$. We have
$G(\theta^*) = 0$ and $G(\hat\theta_t) = \sum_{s=1}^t \epsilon_s X_s$, where $\epsilon_s$ is the sub-Gaussian noise at round $s$. For convenience, define $Z:=G(\hat\theta_t)$. From mean value theorem, for any $\theta_1, \theta_2$, there exists $v\in(0,1)$ and $\bar\theta = v\theta_1 + (1-v) \theta_2$ such that
\begin{equation}\label{gtheta}
    G(\theta_1) - G(\theta_2) = \left[ \sum_{s=1}^t \mu^\prime(X_s^T \bar\theta)  X_s X_s^T \right] (\theta_1 - \theta_2) := F(\bar\theta) (\theta_1 - \theta_2),
\end{equation}
where $F(\bar\theta) = \sum_{s=1}^t \mu^\prime(X_s^T \bar\theta)  X_s X_s^T$.
Therefore, for any $\theta_1\neq \theta_2$, we have
\begin{equation*}
    (\theta_1-\theta_2)^T (G(\theta_1) - G(\theta_2)) = (\theta_1-\theta_2)^T F(\bar\theta) (\theta_1-\theta_2) > 0,
\end{equation*}
since $\mu^\prime > 0$ and $\lambda_{\min}(V_{t+1})>0$. So $G(\theta)$ is an injection from $\mathbbm{R}^d$ to $\mathbbm{R}^d$. 
Consider an $\eta$-neighborhood of $\theta^*$, $\mathbbm{B}_{\eta} := \{ \theta: \|\theta-\theta^*\| \leq \eta\}$, where $\eta$ is a constant that will be specified later such that we have $c_{\eta} = \inf_{\theta\in \mathbbm{B}_{\eta}} \mu^\prime(x^T\theta) > 0$.
When $\theta_1, \theta_2 \in \mathbbm{B}_{\eta}$, from the property of convex set, we have $\bar\theta \in \mathbbm{B}_\eta$.
From Equation \ref{gtheta}, we have when $\theta \in \mathbbm{B}_{\eta}$,
\begin{align*}
   \|G(\theta)\|_{V_{t+1}^{-1}} & = \|G(\theta) - G(\theta^*)\|_{V_{t+1}^{-1}} = \sqrt{ (\theta - \theta^*)^T F(\bar\theta) V_{t+1}^{-1} F(\bar\theta) (\theta - \theta^*)}  \\
   & \geq c_{\eta} \sqrt{\lambda_{\min}(V_{t+1})}  \|\theta - \theta^*\|
\end{align*}
The last inequality is due to 
\begin{equation*}
    F(\bar\theta) \succeq c_{\eta} \sum_{s=1}^t  X_s X_s^T = c_{\eta} V_{t+1}.
\end{equation*}

From Lemma A in \cite{chen1999strong}, we have that 
\begin{equation*}
    \left\{ \theta: \|G(\theta) - G(\theta^*)\|_{V_{t+1}^{-1}} \leq c_{\eta} \eta \sqrt{\lambda_{\min} (V_{t+1})} \right \} \subset \mathbbm{B}_{\eta}.
\end{equation*}
Now %we need to show that 
from Lemma 7 in \cite{li2017provably}, we have with probability at least $1-\delta$,
\begin{equation*}
    \|G(\hat\theta_t) - G(\theta^*)\|_{V_{t+1}^{-1}} = \|Z\|_{V_{t+1}^{-1}} \leq 4R \sqrt{d + \log \frac{1}{\delta} }.
\end{equation*}

Therefore, when 
\begin{equation*}
    \eta \geq \frac{4R}{c_{\eta}} \sqrt{ \frac{d + \log \frac{1}{\delta}}{\lambda_{\min} (V_{t+1}) }},
\end{equation*}
we have $\hat\theta_t \in \mathbbm{B}_{\eta}$.
Since $c_{\eta} \geq c_1 \geq c_3 > 0$ when $\eta\leq 1$, we have
\begin{equation*}
    \|\hat\theta_t - \theta^*\| \leq \frac{4R}{c_{\eta}} \sqrt{ \frac{d + \log \frac{1}{\delta}}{\lambda_{\min} (V_{t+1}) }} \leq 1,
\end{equation*}
when $\lambda_{\min} (V_{t+1}) \geq \frac{16 R^2 [d+\log (\frac{1}{\delta})]}{c_1^2}.$
\end{proof}

\subsection{Proof of Lemma \ref{convergence_sgd}}
Note that the condition of Lemma \ref{close} 
holds with high probability when $\tau$ is chosen as Equation \ref{tau}. This is a consequence of Proposition 1 in \cite{li2017provably}, which is presented below for reader's convenience.
\begin{proposition}
[Proposition 1 in \cite{li2017provably}]
\label{prop}
Define $V_{n+1} = \sum_{t=1}^n X_tX_t^T$, where $X_t$ is drawn IID from some distribution in unit ball $\mathbbm{B}^d$. Furthermore, let $\Sigma := E[X_t X_t^T]$ be the second moment matrix, let $B, \delta_2>0$ be two positive constants. Then there exists positive, universal constants $C_1$ and $C_2$ such that $\lambda_{\min} (V_{n+1}) \geq B$ with probability at least $1-\delta_2$, as long as
\begin{equation*}\label{condition_for_tau}
    n \geq \left( \frac{C_1 \sqrt{d} + C_2 \sqrt{\log(1/\delta_2)}}{\lambda_{\min} (\Sigma)} \right)^2 + \frac{2B}{\lambda_{\min} (\Sigma)}.
\end{equation*}
\end{proposition}

Now we formally prove Lemma \ref{convergence_sgd}.
\begin{proof}
Note that from the definition of $\tilde\theta_0$ in the algorithm, when $j=1$, the conclusion holds trivially.
When $\tau$ is chosen as in Equation \ref{tau}, we have from Lemma 1 and Proposition \ref{prop} that $\|\hat\theta_t - \theta^*\| \leq 1$ for all $t\geq \tau$ with probability at least $1-\frac{2}{T^2}$. Therefore, $\hat\theta_{j\tau} \in \mathcal{C}$ for all $j\geq 1$ with probability at least $1-\frac{2}{T^2}$. 
For the analysis below, we assume $\hat\theta_{j\tau} \in \mathcal{C}$ for all $j\geq 1$.

Since $\tilde\theta_j \in \mathcal{C}$, we have $\| \tilde\theta_j - \theta^*\| \leq 3$. Denote $\mathbbm{B}_{\eta} := \{ \theta: \|\theta-\theta^*\| \leq \eta\}$, we have $\tilde\theta_j, \hat\theta_{j\tau} \in \mathbbm{B}_3$. 
For any $v>0$, define $\bar\theta = v\tilde\theta_j + (1-v)\hat\theta_{j\tau}$, since $\mathbbm{B}_3$ is convex, we have $\bar\theta \in \mathbbm{B}_3$. Therefore, we have from Assumption \ref{mu_prime} 
\begin{equation*}
    \nabla^2 l_{j,\tau} (\bar\theta) = \sum_{s=(j-1)\tau + 1}^{j\tau} \mu^\prime (X_s^T \bar\theta) X_s X_s^T \succeq c_3 \sum_{s=(j-1)\tau + 1}^{j\tau} X_s X_s^T.
\end{equation*}
Since we update $\tilde\theta_j$ every $\tau$ rounds and $\theta^{\text{TS}}_j$ only depends on $\tilde\theta_{j}$. For the next $\tau$ rounds, the pulled arms are only dependent on $\theta^{\text{TS}}_j$. 
Therefore, the feature vectors of pulled arms among the next $\tau$ rounds are IID. 
According to Proposition \ref{prop} and Equation \ref{tau}, and by applying a union bound, we have $\lambda_{\min} \left(\sum_{s=(j-1)\tau + 1}^{j\tau} X_s X_s^T \right) \geq \frac{\alpha}{c_3}$ holds for all $j\geq 1$ with probability at least $1-\frac{1}{T^2}$. This tells us that for all $j$, $l_{j,\tau} (\theta)$ is a $\alpha$-strongly convex function when $\theta \in \mathbbm{B}_3$. Therefore, we can apply (Theorem 3.3 of Section 3.3.1 in \cite{hazan2016introduction}) to get for all $j\geq 1$
\begin{equation*}
    \sum_{q=1}^j \left( l_{q,\tau} (\tilde\theta_q) - l_{q,\tau} (\hat\theta_{j\tau})  \right) \leq \frac{G^2}{2\alpha} (1+\log j)
\end{equation*}
where $G$ satisfies $G^2 \geq E\|\nabla l_{q,\tau}\|^2$. Note that $G \leq \tau$ since $\mu(x) \in [0,1], Y_s \in [0,1]$ and $\|X_s\| \leq 1$. From Jensen's Inequality,
we have 
\begin{equation*}
    \sum_{q=1}^j \left( l_{q,\tau} (\bar\theta_j)- l_{q,\tau} (\hat\theta_{j\tau}) \right) \leq \frac{G^2}{2\alpha} (1+\log j).
\end{equation*}
Since $\bar\theta_j, \hat\theta_{j\tau} \in \mathbbm{B}_3$, we have for any $v>0$, if $\theta = v\bar\theta_j + (1-v)\hat\theta_{j\tau}$, then $\nabla^2 l_{q,\tau}(\theta) \succeq \alpha I_d$ for all $1\leq q\leq j$. Since $\sum_{q=1}^j \nabla l_{q,\tau} (\hat\theta_{j\tau}) = 0$,
we have
\begin{equation*}
    \| \bar\theta_j - \hat\theta_{j\tau} \| \leq \frac{G}{\alpha} \sqrt{\frac{1+\log j}{j}}.
\end{equation*}
By applying a union bound, we get the conclusion.
\end{proof}

\subsection{Proof of Lemma \ref{concentrate_e1}}
We utilize the concentration property of MLE. Here, we present the analysis of MLE in \cite{li2017provably}.
\begin{lem} 
[Lemma 3 in \cite{li2017provably}]
\label{hat_theta_concentrate}
Suppose $\lambda_{\min} (V_{\tau+1}) \geq 1$. For any $\delta_3 \in (0,1)$, the following event
\begin{equation*}
    \mathcal{E} := \left\{ \|\hat\theta_t -\theta^* \|_{V_{t+1}} \leq \frac{R}{c_1} \sqrt{\frac{d}{2} \log(1+\frac{2t}{d}) + \log \frac{1}{\delta_3} }   \right\}
\end{equation*}
holds for all $t\geq \tau$ with probability at least $1-\delta_3$.
\end{lem}
\begin{proof}
Note that from Proposition \ref{prop}, when $\alpha \geq c_3$,
$\lambda_{\min} (V_{\tau+1}) \geq 1$ holds with probability at least $1-\frac{1}{T^2}$.
The proof of Lemma \ref{concentrate_e1} is simply a combination of Lemma \ref{convergence_sgd} and Lemma \ref{hat_theta_concentrate} by applying a union bound.
\end{proof}

\subsection{Proof of Lemma \ref{ts_concentrate}}
We use formula 7.1.13 in \cite{abramowitz1948handbook} to help derive the concentration and anti-concentration inequalities for Gaussian distributed random variables. Details are shown in Lemma \ref{normal_ineq}.
\begin{lem}[Formula 7.1.13 in \cite{abramowitz1948handbook}]
\label{normal_ineq}
For a Gaussian distributed random variable with mean $m$ and variance $\sigma^2$, we have for $z\geq 1$ that
\begin{equation*}
    \mathbbm{P}( |Z-m| \geq z\sigma)  \leq  \frac{1}{\sqrt{\pi} } e^{-\frac{z^2}{2}}.
\end{equation*}
For $0<z\leq 1$, we have
\begin{equation*}
   \mathbbm{P}( |Z-m| \geq z\sigma)  \geq  \frac{1}{2\sqrt{\pi} } e^{-\frac{z^2}{2}} .
\end{equation*}
\end{lem}
Now we prove Lemma \ref{ts_concentrate}.
\begin{proof}
Since $\theta_j^{\text{TS}} | \mathcal{F}_{j\tau}\sim \mathcal{N} \left (\bar\theta_j, \left(2g_1 (j)^2 \frac{c_3}{\alpha j} + \frac{2g_2(j)^2}{j} \right) I_d  \right)$, and $\theta^{\text{TS}}_j$ is independent of 
$\left\{\cup_{t=j\tau+1}^{(j+1)\tau}\mathcal{A}_{t}\right\} = \{x_{t,a}, a\in [K], j\tau < t \leq (j+1)\tau \}$, we have for $x \in \left\{\cup_{t=j\tau+1}^{(j+1)\tau}\mathcal{A}_{t}\right\}$,
\begin{equation*}
    x^T (\bar\theta_j - \theta_j^{\text{TS}}) |  \mathcal{F}_{j\tau}, x \sim \mathcal{N} \left(0, \left( 2g_1 (j)^2 \frac{c_3}{\alpha j} + \frac{2g_2(j)^2}{j} \right) \|x\|^2 \right).
\end{equation*}
From the property of Gaussian random variable in Lemma \ref{normal_ineq}, when $u = \sqrt{2\log (T^2K\tau)}$, we have
\begin{equation}\label{cp}
    \mathbbm{P} \left(|x^T (\bar\theta_j - \theta_j^{\text{TS}}) | \geq u  \sqrt{ 2g_1 (j)^2 \frac{c_3}{\alpha j}\|x\|^2 + \frac{2g_2(j)^2}{j} \|x\|^2}  \middle | \mathcal{F}_{j\tau}, x\right ) \leq \frac{1}{\sqrt{\pi} } e^{-\frac{u^2}{2}} \leq \frac{1}{K\tau T^2}.
\end{equation}
% Therefore, even though $x$ is random, 
We use the following property of conditional probability
\begin{equation}\label{cp_prop}
    \int_x \mathbbm{P} (E | X=x, \mathcal{F} ) f(X=x|\mathcal{F}) dx = \mathbbm{P} (E | \mathcal{F} ),
\end{equation}
where $f(X=x|\mathcal{F})$ is the conditional \textit{p.d.f} of a random variable $X$ and $E$ is an event.
Combine Equation \ref{cp} and Equation \ref{cp_prop}, we have for every $a\in [K]$ and $j\tau < t \leq (j+1)\tau$,
\begin{align*}
    & \mathbbm{P} \left(|x_{t,a}^T (\bar\theta_j - \theta_j^{\text{TS}}) | \geq u  \sqrt{ 2g_1 (j)^2 \frac{c_3}{\alpha j} + 2g_2(j)^2 / j \|x_{t,a}\|^2}  \middle | \mathcal{F}_{j\tau} \right ) \\
    & = \int_x \mathbbm{P} \left(|x_{t,a}^T (\bar\theta_j - \theta_j^{\text{TS}}) | \geq u  \sqrt{ 2g_1 (j)^2 \frac{c_3}{\alpha j} + 2g_2(j)^2 / j \|x_{t,a}\|^2}  \middle | \mathcal{F}_{j\tau} , x_{t,a} = x\right ) f( x_{t,a}=x|\mathcal{F}_{j\tau} ) dx\\
    & \leq \frac{1}{K\tau T^2} \int_x  f( x_{t,a}=x|\mathcal{F}_{j\tau} ) dx = \frac{1}{K\tau T^2}
\end{align*}
Applying a union bound, we get the conclusion. % holds for all $x \in \{x_{t,a}, a\in [K], j\tau < t \leq (j+1)\tau \}$.
\end{proof}

\subsection{Proof of Lemma \ref{anti_concentrate}}
\begin{proof}
We still use Lemma \ref{normal_ineq} to show the result. For convenience, denote $x := x_{t,*}$, $\gamma_1 := \sqrt{\frac{c_3}{\alpha {j_t}}}\|x\|$ and $\gamma_2 := \frac{\|x\|}{\sqrt{{j_t}}}$. % Then $2g_1 ({j_t})^2 \gamma_1^2 + 2g_2({j_t})^2 \gamma_2^2$ is the variance of $x^T \theta^{\text{TS}}_{j_t}$. 
Note that $x$ is independent of $\theta_{j_t}^{\text{TS}}$, so
\begin{equation}
    x^T (\bar\theta_{j_t} - \theta_{j_t}^{\text{TS}}) |  \mathcal{F}_{j_t\tau}, x \sim \mathcal{N} \left(0, \left( 2g_1 ({j_t})^2 \gamma_1^2 + 2g_2({j_t})^2 \gamma_2^2 \right)  \right).
\end{equation}
Therefore, 
\begin{align*}
    & \mathbbm{P}\left( x^T \theta^{\text{TS}}_{j_t}  > x^T \theta^* \middle | \mathcal{F}_{j_t\tau}, x \right) 
    = \mathbbm{P}\left( \frac{ x^T \theta^{\text{TS}}_{j_t} - x^T \bar\theta_{j_t}}{ \sqrt{ 2g_1 ({j_t})^2 \gamma_1^2 + 2g_2({j_t})^2 \gamma_2^2}  }   > \frac{ x^T \theta^* - x^T \bar\theta_{j_t}}{\sqrt{ 2g_1 ({j_t})^2 \gamma_1^2 + 2g_2({j_t})^2 \gamma_2^2} }  \middle | \mathcal{F}_{{j_t}\tau}, x \right) \\
    & \geq  \mathbbm{P}\left( \frac{ x^T \theta^{\text{TS}}_{j_t} - x^T \bar\theta_{j_t}}{\sqrt{ 2g_1 ({j_t})^2 \gamma_1^2 + 2g_2({j_t})^2 \gamma_2^2}  }   > \frac{ g_1 ({j_t}) \|x\|_{V_{{j_t}\tau+1}^{-1}} + g_2({j_t}) \frac{\|x\|}{\sqrt{j_t}}}{\sqrt{ 2g_1 ({j_t})^2 \gamma_1^2 + 2g_2({j_t})^2 \gamma_2^2} }  \middle | \mathcal{F}_{{j_t}\tau}, x \right)  \\ 
    & \geq \mathbbm{P}\left( \frac{ x^T \theta^{\text{TS}}_{j_t} - x^T \bar\theta_{j_t}}{\sqrt{ 2g_1 ({j_t})^2 \gamma_1^2 + 2g_2({j_t})^2 \gamma_2^2} }   > \frac{ g_1 ({j_t}) \sqrt{\frac{c_3}{\alpha {j_t}}}\|x\| + g_2(j) \frac{\|x\|}{\sqrt{{j_t}}}}{\sqrt{ 2g_1 ({j_t})^2 \gamma_1^2 + 2g_2({j_t})^2 \gamma_2^2} }  \middle | \mathcal{F}_{{j_t}\tau} , x\right)  \\ 
    & \geq \frac{1}{4\sqrt{\pi} } e^{-\frac{z^2}{2}} , % -\frac{6}{T},
\end{align*}
where $z : = \frac{ g_1 ({j_t}) \gamma_1 + g_2({j_t}) \gamma_2}{\sqrt{ 2g_1 ({j_t})^2 \gamma_1^2 + 2g_2({j_t})^2 \gamma_2^2}}$.
The first and second inequalities hold since $\mathcal{F}_t$ is a filtration such that $E_1({j_t})$ and
$\lambda_{\min} (V_{{j_t}\tau+1}) \geq \frac{\alpha {j_t}}{c_3}$ are true.
Notice that we have $0<z\leq 1$ since
\begin{equation*}
2g_1 ({j_t})^2 \gamma_1^2 + 2g_2({j_t})^2 \gamma_2^2 - (g_1 ({j_t}) \gamma_1 + g_2({j_t}) \gamma_2)^2 = (g_1 ({j_t}) \gamma_1 - g_2({j_t}) \gamma_2)^2 \geq 0.
\end{equation*}
Therefore, we get
\begin{equation*}
    \mathbbm{P}\left( x^T \theta^{\text{TS}}_{j_t}  > x^T \theta^* \middle | \mathcal{F}_{{j_t}\tau}, x \right) \geq \frac{1}{4\sqrt{\pi} } e^{-\frac{z^2}{2}} \geq \frac{1}{4\sqrt{\pi e}}. 
\end{equation*}
Similarly, using Equation \ref{cp_prop}, we get 
\begin{align*}
    \mathbbm{P}\left( x_{t,*}^T \theta^{\text{TS}}_{j_t}  > x_{t,*}^T \theta^* \middle | \mathcal{F}_{j_t\tau} \right) = \int_x \mathbbm{P}\left( x_{t,*}^T \theta^{\text{TS}}_{j_t}  > x_{t,*}^T \theta^* \middle | \mathcal{F}_{j_t\tau}, x_{t,*} = x \right) f(x_{t,*} = x |\mathcal{F}_{j_t\tau}) dx 
    \geq \frac{1}{4\sqrt{\pi e}}. 
\end{align*}
\end{proof}

\subsection{Proof of Lemma \ref{reg_t}}
The technique used in this proof is extracted from \cite{agrawal2013thompson,kveton2019perturbed}.
\begin{proof}
Denote $\mathbbm{E}_t [\cdot] := \mathbbm{E} [\cdot | \mathcal{F}_{t}]$. To prove the lemma, we prove the following Equation \ref{key} holds for any possible filtration $\mathcal{F}_t$:
\begin{equation}\label{key}
    \mathbbm{E}_{j_t\tau} [
    \Delta_{a_t} (t) \mathbbm{1}(E_1(j_t) \cap E_2(j_t) \cap E_3(j_t)) ] \leq \left(1 + \frac{2}{\frac{1}{4\sqrt{\pi e}} - \frac{1}{T^2}} \right) \mathbbm{E}_{j_t\tau} \left[  H_{a_t} (t) \mathbbm{1}(E_3(j_t) )  \right ] 
\end{equation}
Denote the following set as the underesampled arms at round $t$, 
\begin{align*}
    S_t^C  & = \left\{ i\in [K]: H_i (t)  \geq \Delta_i(t) \right\}
\end{align*}
Note that $a_t^* \in S_t^C$ for all $t$. The set of sufficiently sampled arms is $S_t = [K] \setminus S_t^C$.
Let $J_t = \argmin_{i \in S_t^C} H_i (t)$ be the least uncertain undersampled arm at round $t$. At round $t$, denote $j_t = \lfloor \frac{t-1}{\tau} \rfloor$. In the steps below, we assume that event $E_1(j_t) \cap E_2 (j_t)$ occurs, then
\begin{align*} 
    \Delta_{a_t} (t) & = \Delta_{J_t} (t) + (x_{t,J_t} - X_t)^T \theta^* \\
    & = \Delta_{J_t} (t) + x_{t,J_t}^T ( \theta^* - \theta_{j_t}^{TS} ) + (x_{t,J_t} - X_t)^T \theta_{j_t}^{TS} + X_t^T (\theta_{j_t}^{TS} - \theta^*) \\
    & \leq \Delta_{J_t} (t) + H_{J_t} (t) + H_{a_t} (t) \quad \text{ since $(x_{t,J_t} - X_t)^T \theta_{j_t}^{TS} \leq 0$} \\
    & \leq 2 H_{J_t} (t) + H_{a_t} (t) \quad \text{ since $J_t \in S_t^C$}.
\end{align*}

The left to do is to bound $H_{J_t} (t)$ by $H_{a_t} (t)$. 
% Note that $H_i(t)$ only depends on the information up to round $j_t \tau$. So $S_t^C$ is $\mathcal{F}_{j_t\tau}$-measurable. 
Since $J_t = \argmin_{i \in S_t^C} H_i (t)$, we have
\begin{equation}\label{bound_jt}
    \mathbbm{E}_{j_t\tau} \left[ H_{a_t} (t)  \right ] \geq  \mathbbm{E}_{j_t\tau} \left[ H_{a_t} (t)  | a_t \in S_t^C \right ] \mathbbm{P}\left( a_t \in S_t^C | \mathcal{F}_{j_t\tau} \right) \geq \mathbbm{E}_{j_t\tau} [H_{J_t} (t)]  \mathbbm{P}\left( a_t \in S_t^C | \mathcal{F}_{j_t\tau} \right).
\end{equation}
Therefore, we have
\begin{equation}\label{fit}
    \mathbbm{E}_{j_t\tau} \left[\Delta_{a_t} (t) \mathbbm{1}(E_1(j_t) \cap E_2(j_t)) \right] \leq \left(1+\frac{2}{P\left( a_t \in S_t^C | \mathcal{F}_{j_t\tau}\right)} \right) \mathbbm{E}_{j_t\tau} \left[ H_{a_t} (t)  \right ]
\end{equation}
Next, we bound $P\left( a_t \in S_t^C | \mathcal{F}_{j_t\tau} \right)$. 
\begin{align}
    & \mathbbm{P}\left( a_t \in S_t^C | \mathcal{F}_{j_t\tau} \right) \geq \mathbbm{P}\left( x_{t,*}^T \theta^{\text{TS}}_{j_t} \geq \max_{i\in S_t} x_{t,i}^T \theta^{\text{TS}}_{j_t}  \middle | \mathcal{F}_{j_t\tau} \right) 
    \quad \text{ since $a_t^* \in S_t^C$}  \nonumber \\
    & \geq \mathbbm{P}\left( x_{t,*}^T \theta^{\text{TS}}_{j_t} \geq \max_{i\in S_t} x_{t,i}^T \theta^{\text{TS}}_{j_t}, E_1(j_t)\cap E_2(j_t)  \middle | \mathcal{F}_{j_t\tau} \right) \nonumber \\
    & \geq \mathbbm{P}\left( x_{t,*}^T \theta^{\text{TS}}_{j_t} \geq  x_{t,*}^T \theta^*, E_1(j_t)\cap E_2(j_t)  | \mathcal{F}_{j_t\tau} \right) \label{third} \\
    & \geq \mathbbm{P}\left( x_{t,*}^T \theta^{\text{TS}}_{j_t} \geq  x_{t,*}^T \theta^*, E_1(j_t) | \mathcal{F}_{j_t\tau} \right) - \mathbbm{P}\left( E_2^C(j_t) | \mathcal{F}_{j_t\tau} \right) \nonumber \\
    & \geq \mathbbm{P}\left( x_{t,*}^T \theta^{\text{TS}}_{j_t} \geq  x_{t,*}^T \theta^*, E_1(j_t) | \mathcal{F}_{j_t\tau} \right) - \frac{1}{T^2} \label{last}.   
\end{align}
Inequality \ref{third} holds because for all $i\in S_t$, on event $E_1(j_t)\cap E_2(j_t)$, 
\begin{equation*}
    x_{t,i}^T\theta^{\text{TS}}_{j_t} \leq x_{t,i}^T \theta^* + H_i (t) < x_{t,i}^T \theta^* + \Delta_i (t) = x_{t,*}^T \theta^*.
\end{equation*}
Inequality \ref{last} holds because of Lemma \ref{ts_concentrate}.
When $\mathcal{F}_t$ is a filtration such that $E_1(j_t)$ and $E_3(j_t)$ are true, we have from Lemma \ref{anti_concentrate} that 
\begin{equation*}
    \mathbbm{P}\left( a_t \in S_t^C | \mathcal{F}_{j_t\tau} \right) \geq \frac{1}{4\sqrt{\pi e}} - \frac{1}{T^2}.
\end{equation*}
So under such filtration, from Equation \ref{fit}, we have 
\begin{equation*}
    \mathbbm{E}_{j_t\tau} \left[
    \Delta_{a_t} (t) \mathbbm{1}(E_1(j_t) \cap E_2(j_t)  ) 
    \right] 
    \leq \left(1+\frac{2}{\frac{1}{4\sqrt{\pi e}} - \frac{1}{T^2}} \right) \mathbbm{E}_{j_t\tau} \left[ H_{a_t} (t)  \right ].
\end{equation*}
Since $E_3(j_t)$ is $\mathcal{F}_{j_t\tau}$-measurable, we have under such filtration,
\begin{equation*}
    \mathbbm{E}_{j_t\tau} [
    \Delta_{a_t} (t) \mathbbm{1}(E_1(j_t) \cap E_2(j_t) \cap E_3(j_t) ) ]
    \leq \left(1 + \frac{2}{\frac{1}{4\sqrt{\pi e}} - \frac{1}{T^2}} \right) \mathbbm{E}_{j_t\tau} \left[  H_{a_t} (t) \mathbbm{1} (E_3(j_t))\right ].
\end{equation*}
When $\mathcal{F}_t$ is a filtration such that $E_1(j_t)\cap E_3(j_t)$ is not true, the conclusion holds trivially. This finishes our proof. 
\end{proof}

\subsection{Proof of Theorem \ref{main_thm}}
Before proving the theorem, we show a lemma below.

\begin{lem}\label{bound_ht}
Let $J = \lfloor \frac{T}{\tau} \rfloor$, then
 \begin{align*}
     \mathbbm{E}\left[\sum_{t=\tau+1}^T H_{a_t} (t) \mathbbm{1}(E_3(j_t)) \right]
    & \leq \sqrt{\tau T} \left(2 g_1(J) \sqrt{\frac{c_3 }{\alpha}} + 2 g_2(J) + u \sqrt{ 2g_1(J)^2 \frac{c_3 }{\alpha}  + 2g_2(J)^2  } \sqrt{1+\log J} \right).
 \end{align*}
\end{lem}
\begin{proof}
We know $H_{a_t}(t) = H_{a_t,1} (t) + H_{a_t,2} (t) + H_{a_t,3} (t)$ from definition, where
\begin{align*}
    H_{i,1} (t) &= g_1 (j_t) \|x_{t,i}\|_{V_{j_t\tau+1}^{-1}}, \quad
    H_{i,2} (t) = g_2 (j_t) \frac{\|x_{t,i}\|}{\sqrt{j_t}}, \\
    H_{i,3} (t) & = u \sqrt{ 2g_1 (j_t)^2 \frac{c_3}{\alpha j_t}\|x_{t,i}\|^2 + 2g_2(j_t)^2 \frac{\|x_{t,i}\|^2}{j_t} }
\end{align*}
For all $t$, we have $j_t \leq \lfloor\frac{T}{\tau} \rfloor$ and so
$g_1(j_t) \leq g_1(J)$, and 
$g_2 (j_t) \leq g_2 (J) $. 
Since $\|X_t\|^2_{V_{j\tau+1}^{-1}} \leq \lambda_{\max} (V_{j\tau+1}^{-1}) \|X_t\|^2 \leq \frac{c_3}{\alpha j}$ when $E_3(j)$ holds, we have
\begin{equation}\label{h1}
    \mathbbm{E}\left[\sum_{t=\tau+1}^T H_{a_t,1} (t)  \mathbbm{1}(E_3(j_t)) \right]
    \leq 2\tau g_1(J) \sqrt{\frac{c_3}{\alpha}J}  \leq 2 g_1(J) \sqrt{\frac{c_3 \tau}{\alpha}}\sqrt{T}. 
\end{equation}
We also have
\begin{equation}\label{h2}
    \sum_{t=\tau+1}^T H_{a_t,2} (t) \leq  g_2(J) \sum_{t=\tau+1}^T \frac{\|X_t\|}{\sqrt{j_t}} 
    \leq 2 g_2(J)\sqrt{\tau T}.
\end{equation}
From Cauchy-Schwarz, we have 
\begin{align}\label{h3}
    \sum_{t=\tau+1}^T H_{a_t,3} (t) & \leq  u \sqrt{T} \sqrt{ \sum_{t=\tau+1}^T 2g_1 (j_t)^2 \frac{c_3}{\alpha j_t}\|X_t\|^2 + 2g_2(j_t)^2 \frac{\|X_t\|^2}{j_t} } \nonumber \\
    & \leq u \sqrt{T} \sqrt{ 2g_1(J)^2 \frac{c_3 \tau}{\alpha} (1+\log J) + 2g_2(J)^2 \tau (1+\log J) }.
\end{align}
Combine Equation \ref{h1}, \ref{h2}, \ref{h3}, we get the conclusion.
\end{proof}

Now we formally prove Theorem \ref{main_thm}.
\begin{proof}
Since
\begin{align*}
    & \mathbbm{E}_{j_t\tau} \left[\mu(x_{t,*}^T \theta^*) - \mu(X_t^T \theta^*) \right]  
    \leq \mathbbm{E}_{j_t\tau} \left[\left(\mu(x_{t,*}^T \theta^*) - \mu(X_t^T \theta^*)\right) \mathbbm{1}(E_2(j_t))\right]  + \mathbbm{P}(E_2^C(j_t) | \mathcal{F}_{j_t\tau}) \\
    & \leq \mathbbm{E}_{j_t\tau} \left[\left(\mu(x_{t,*}^T \theta^*) - \mu(X_t^T \theta^*) \right)\mathbbm{1}(E_2(j_t))\right]  + \frac{1}{T^2},
\end{align*}
we have
\begin{align*}
     \mathbbm{E} \left[\mu(x_{t,*}^T \theta^*) - \mu(X_t^T \theta^*) \right]  
    \leq \mathbbm{E} \left[\left(\mu(x_{t,*}^T \theta^*) - \mu(X_t^T \theta^*) \right)\mathbbm{1}(E_2(j_t))\right]  + \frac{1}{T^2}
\end{align*}
From Proposition \ref{prop}, when $\tau$ is chosen as in Equation \ref{tau}, $E_3(j_t)$ holds with probability with at least $1-\frac{1}{T^2}$ for every $t$.
From the above,
\begin{align*}
    & \mathbbm{E}[R(T)]  =  \sum_{t=1}^T \mathbbm{E} \left[\mu(x_{t,*}^T \theta^*) - \mu(X_t^T \theta^*) \right]  
    \leq \sum_{t=1}^T \mathbbm{E} \left[\left(\mu(x_{t,*}^T \theta^*) - \mu(X_t^T \theta^*) \right)\mathbbm{1}(E_2(j_t))\right]  + \frac{1}{T}\\
    & \leq \mathbbm{E} \left[\sum_{t=1}^T \left(\mu(x_{t,*}^T \theta^*) - \mu(X_t^T \theta^*) \right)  \mathbbm{1} (E_1(j_t) \cap E_2(j_t) \cap E_3(j_t)) \right] + \sum_{t=1}^T \mathbbm{P} (E_1^C(j_t) \cup E_3^C(j_t)) + \frac{1}{T}\\
    & \leq \tau + L_{\mu} \sum_{t=\tau+1}^T \mathbbm{E}[\Delta_{a_t} (t) \mathbbm{1}(E_1(j_t) \cap E_2(j_t) \cap E_3(j_t))] + \frac{7}{T} \\
    & \leq \tau + pL_{\mu} \sum_{t=\tau+1}^T \mathbbm{E} \left[  H_{a_t} (t)  \mathbbm{1}(E_3(j_t))  \right ]  + \frac{7}{T} \quad \text{ from Lemma \ref{reg_t}.}
\end{align*}
From Lemma \ref{bound_ht}, we have
\begin{equation*}
    \mathbbm{E}[R(T)] \leq \tau + L_{\mu} p \sqrt{\tau T} \left[2\sqrt{\frac{c_3}{\alpha}} g_1(J) + 2 g_2(J) + u \sqrt{ \frac{2c_3 g_1(J)^2}{\alpha} + 2g_2(J)^2} \sqrt{1+\log \lfloor \frac{T}{\tau} \rfloor}\right] + \frac{7}{T}.
\end{equation*}
This ends our proof.
\end{proof}

\subsection{Discussion}\label{d}
As pointed out by the reader, since $\|x_{t,a}\| \leq 1$, so $\sigma_0^2, \lambda_f \leq O(\frac{1}{d})$. So a more realistic assumption should be $\sigma_0^2,\lambda_f \sim O(\frac{1}{d})$. However, we found that $\sigma_0^2 \sim O(1)$ is an assumption that is widely used in literature (see \cite{li2017provably}). If we assume $\sigma_0^2, \lambda_f \sim O(1/d)$, then the regret upper bound of our algorithm is $\mathbbm{E}[R(T)] \leq \tilde O(d^{\frac{5}{2}} \sqrt{T})$ and the regret upper bound of UCB-GLM \citep{li2017provably} is $\tilde O(d^3 + d \sqrt{T})$.

\end{document}